\PassOptionsToPackage{prologue,dvipsnames}{xcolor}
\documentclass[acmsmall]{acmart}
\usepackage{array}
\usepackage{amsmath,amsfonts}
\usepackage{algorithmic}
\usepackage[ruled]{algorithm2e}
\usepackage{array}
\usepackage{appendix}
\usepackage[caption=false,font=normalsize,labelfont=sf,textfont=sf]{subfig}
\usepackage{textcomp}
\usepackage{stfloats}
\usepackage{url}
\usepackage{verbatim}
\usepackage{graphicx}
\usepackage{amsthm}

\usepackage{amssymb}
\usepackage{bm}
\usepackage[normalem]{ulem}

\usepackage{xcolor}
\usepackage{stfloats}
\usepackage{multirow}
\usepackage{makecell}
\usepackage{bbding}
\usepackage{float}
\usepackage[misc]{ifsym}
\newtheorem{theorem}{Theorem}[]
\useunder{\uline}{\ul}{}
\AtBeginDocument{%
  \providecommand\BibTeX{{%
    \normalfont B\kern-0.5em{\scshape i\kern-0.25em b}\kern-0.8em\TeX}}}

\acmJournal{TIST}

\begin{document}

\title{Exploring the Distributed Knowledge Congruence in Proxy-data-free Federated Distillation}


\author{Zhiyuan Wu}
\email{wuzhiyuan22s@ict.ac.cn}
\affiliation{%
  \institution{Institute of Computing Technology, Chinese Academy of Sciences}
  \streetaddress{No.6, Academy of Science South Road, Zhongguancun, Haidian District}
  \city{Beijing}
  \country{China}
  \postcode{100086}
}
\affiliation{%
  \institution{University of Chinese Academy of Sciences}
  \streetaddress{No.19, Yuquan Road, Shijingshan District, Beijing}
  \city{Beijing}
  \country{China}
  \postcode{100049}
}

\author{Sheng Sun}
\email{sunsheng@ict.ac.cn}
\author{Yuwei Wang}
\authornote{Corresponding author.}
\email{ywwang@ict.ac.cn}
\author{Min Liu}
\email{liumin@ict.ac.cn}
\author{Quyang Pan}
\email{lightinshadow1110@gmail.com}
\affiliation{%
  \institution{Institute of Computing Technology, Chinese Academy of Sciences}
  \streetaddress{No.6, Academy of Science South Road, Zhongguancun, Haidian District}
  \city{Beijing}
  \country{China}
  \postcode{100086}
}

\author{Junbo Zhang}
\email{msjunbozhang@outlook.com}
\affiliation{%
  \institution{JD iCity, JD Technology}
  \city{Beijing}
  \country{China}
}

\affiliation{%
  \institution{JD Intelligent Cities Research}
  \country{China}
}

 \quad

\author{Zeju Li}
\email{lizeju0727@gmail.com}
\affiliation{%
  \institution{Beijing University of Posts and Telecommunications}
  \city{Beijing}
  \country{China}
}

\author{Qingxiang Liu}
\email{qingxiangliu737@gmail.com}
\affiliation{%
  \institution{Institute of Computing Technology, Chinese Academy of Sciences}
  \streetaddress{No.6, Academy of Science South Road, Zhongguancun, Haidian District}
  \city{Beijing}
  \country{China}
  \postcode{100086}
}
\affiliation{%
  \institution{University of Chinese Academy of Sciences}
  \streetaddress{No.19, Yuquan Road, Shijingshan District}
  \city{Beijing}
  \country{China}
  \postcode{100049}
}

\thanks{This work was supported by the National Key Research and Development Program of China (2021YFB2900102), the National Natural Science Foundation of China (62072436), the Beijing Natural Science Foundation (4212021), and the Beijing Science and Technology Project (Z211100004121008) }

\renewcommand{\shortauthors}{Z. Wu et al.}


\begin{abstract}
  Federated learning (FL) is a privacy-preserving machine learning paradigm in which the server periodically aggregates local model parameters from clients without assembling their private data.
    Constrained communication and personalization requirements pose severe challenges to FL. Federated distillation (FD) is proposed to simultaneously address the above two problems, which exchanges knowledge between the server and clients, supporting heterogeneous local models while significantly reducing communication overhead. However, most existing FD methods require a proxy dataset, which is often unavailable in reality. 
    A few recent proxy-data-free FD approaches can eliminate the need for additional public data, but suffer from remarkable discrepancy among local knowledge due to client-side model heterogeneity, leading to ambiguous representation on the server and inevitable accuracy degradation.
    To tackle this issue, we propose a proxy-data-free FD algorithm based on distributed knowledge congruence (FedDKC). FedDKC leverages well-designed refinement strategies to narrow local knowledge differences into an acceptable upper bound, so as to mitigate the negative effects of knowledge incongruence. 
    Specifically, from perspectives of peak probability and Shannon entropy of local knowledge, we design kernel-based knowledge refinement (KKR) and searching-based knowledge refinement (SKR) respectively, and theoretically guarantee that the refined-local knowledge can satisfy an approximately-similar distribution and be regarded as congruent.
    Extensive experiments conducted on three common datasets demonstrate that our proposed FedDKC significantly outperforms the state-of-the-art on various heterogeneous settings while evidently improving the convergence speed.
\end{abstract}


\maketitle

\section{Introduction}
\label{introduction}
Federated learning (FL) is a privacy-preserving machine learning paradigm that allows participants to collaboratively train machine learning (ML) models while keeping the data decentralized.
Owing to the advantages of protecting data privacy and boosting model accuracy, FL has been widely applied to a variety of applications, such as medical treatment \cite{medical1,medical2}, financial risk management \cite{kawa2019credit}, and recommendation systems \cite{tan2020federated,jalalirad2019simple}.
Conventional parameter-aggregation-based FL frameworks \cite{mcmahan2017communication,li2020federated} periodically aggregate local model parameters uploaded by distributed clients on the server-side and then broadcast the updated global model to clients until model convergence, aiming to improve the trained models' generic performance. 
However, such methods face two challenges to tackle.
On the one hand, frequently exchanging model parameters over the training process leads to an excessive communication burden;
on the other hand, homogeneous models among clients conflict with client heterogeneity in terms of data distribution and system configuration. 
The above-mentioned defects easily result in drastic performance drops and hinder the actual deployment of FL.

Motivated by the challenges above, federated distillation (FD) is proposed via extending knowledge distillation technology into FL frameworks \cite{hinton2015distilling,anil2018large}, in which model outputs (called knowledge) in replacement of model parameters are exchanged between clients and the server.
Since the size of knowledge is smaller than that of model weights by many orders of magnitude and knowledge is independent of model architectures, FD can maintain low communication overhead while allowing to design personalized models for individual clients, which is deemed as a communication-efficient and heterogeneous-allowable FL paradigm.

Most existing FD methods \cite{lin2020ensemble,cheng2021fedgems,itahara2020distillation} require a globally-shared proxy dataset to extract knowledge, based on which the server and clients can conduct co-distillation to narrow their representation gap.
Since the proxy dataset needs to be cautiously gathered and is not available in reality, a few efforts are made to explore FD frameworks in a proxy-data-free manner.
Typical methods break the dependence on proxy data via iteratively exchanging additional information between clients and the server, such as a generator \cite{zhu2021data} or local-global models \cite{lee2021preservation,pan2021global} to realize distillation.
However, such methods remarkably increase communication overhead because of exchanging model parameters.
In order to maintain communication efficiency in proxy-data-free FD, \cite{he2020group} proposes a novel feature-driven FD framework, which leverages embedded features in the replacement of proxy data to extract knowledge and achieve workable client-server co-distillation with little influence on communication efficiency. 
Nevertheless, this approach suffers from a non-negligible problem of accuracy degradation, 
since heterogeneous local models tend to exhibit a significant knowledge discrepancy without the assistance of a proxy dataset.
Such knowledge incongruence will lead to unstable and incorrect distillation, which is undoubtedly harmful for the FD process.


To alleviate the accuracy drop caused by knowledge discrepancy among clients, we investigate proxy-data-free FD from a novel perspective: refinement-based distributed knowledge congruence among heterogeneous clients.
We propose a feature-driven FD algorithm based on distributed knowledge congruence (namely FedDKC), in which we refine distributed local knowledge from clients to satisfy a similar distribution based on our well-designed congruence-refinement strategies during server-side distillation. 
Specifically, we consider knowledge discrepancy from two perspectives: the peak probability and the Shannon entropy of knowledge, and propose kernel-based knowledge refinement (KKR) and searching-based knowledge refinement (SKR) strategies, respectively.
On this foundation, the server can learn unbiased knowledge representations and obtain more precise global knowledge based on relatively-congruent local knowledge. In turn, clients can achieve better performance with transferred global knowledge.
As far as we know, \textbf{this paper is the first work to consider knowledge incongruence among heterogeneous clients in proxy-data-free federated distillation.} Our proposed FedDKC can significantly boost training accuracy while maintaining communication efficiency based on distributed knowledge congruence.

The main contributions of this paper are summarized as follows:
\begin{itemize}
	\item
	We propose a communication-efficient and accuracy-guaranteed FD algorithm (namely FedDKC), where local knowledge discrepancy among clients with heterogeneous models is narrowed through skillfully refining to a similar probability distribution.
	In FedDKC, the server can learn unbiased knowledge representation and help clients promote local training accuracy. 
	\item
	We design KKR and SKR strategies severally for two kinds of knowledge incongruence.
	The KKR strategy refines the peak probability of clients' local knowledge to the given limitation, and the SKR strategy makes the Shannon entropy of the refined-local knowledge not exceed the target range.
	We further prove that the knowledge discrepancy between arbitrating clients satisfies an acceptable theoretical upper bound when adopting both strategies.
	\item
    We conduct empirical experiments on MNIST, CIFAR-10, and CINIC-10 datasets with heterogeneous client model architectures and multiple data Non-IID settings. Results demonstrate that our proposed FedDKC outperforms the related state-of-the-art with accuracy improvements and faster convergence on individual clients.
\end{itemize}

\begin{table}[!t]
	\centering
	\setlength\extrarowheight{1.3pt}
	\caption{Main notations with descriptions.}
	\begin{tabular}{c|c}
		\hline
		\textbf{Notation} & \textbf{Description} \\ \hline
		$K$        & The number of clients               \\
        $C$        & The number of classes \\
		$\mathcal{D}^k$      & The private dataset of client $k$     \\
        $(X^k,y^k)$        & The data and labels in $\mathcal{D}^k$\\
        $N^k$      & The number of samples in $\mathcal{D}^k$\\
		$\mathcal{P}$ & The universal set of probability space in $C$ classes\\
		$W^S$ & The global model weights of the server \\
		$W^k$ & The local model weights of client $k$ \\
		$W^k_e$ & The feature extractor weights of client $k$\\
		$W^k_p$ & The predictor weights of client $k$\\
		$z_{X^k}^S$ & The global knowledge\\
		$z^k_{X^k}$ & The local knowledge from client $k$\\
		$p_{X^k}^S$ & The softmax-normalized global knowledge \\
		$p_{X^k}^k$ & The softmax-normalized local knowledge from client $k$\\
		${H^k}$ & The extracted features from client $k$\\
		$\theta$ & The parameter of the auxiliary mapping in $\psi(\theta;\cdot)$\\
		$m$ & The index of the maximum element in $p_{X^k}^k$ \\
		$t$ & The input scaling parameter of kernel function \\
		$T$ & The hyper-parameter of target peak probability\\	
		$E$ & The hyper-parameter of target Shannon entropy  \\
		$\tau(\cdot)$ & The softmax mapping \\
		$L_{CE}(\cdot)$ & The cross-entropy loss function \\
		$L_{sim}(\cdot)$ & The knowledge-similarity loss function \\
		$\max(\cdot)$ & The maximum function \\
		$\varphi(\cdot)$ & The refinement mapping over distributed knowledge \\
		$\sigma ( \cdot )$ & The kernel function in KKR \\
		$H(\cdot)$& The Shannon entropy function \\
		$\psi(\cdot)$ & The auxiliary mapping in SKR\\
		\hline
	\end{tabular}
	\label{main-notation}
\end{table}

\section{Preliminary and Motivation}
This section provides the fundamental process of proxy-data-free FD, and then emphasizes our motivation on distributed knowledge congruence.
Detailed notations and descriptions are given in Table \ref{main-notation}.

\subsection{Basic Process of Proxy-data-free Federated Distillation}
Without loss of generality, we consider the classification task in FL setting with $C$ categories, and let $\mathcal{C}=\{1,2,......,C\}$.
The FD system consists of a large-scale server and $K$ heterogeneous clients.
Let $\mathcal{K}= \{ 1,2,......,K\} $ denote the set of clients.
Each client $k$ owns a private dataset $\mathcal{D}^k=\{{X^k},{y^k}\}$ with $N^k$ samples, where $X^k$ and $y^k$ denotes the set of input data and corresponding labels, respectively. Moreover, data distributions among clients are not identically and independently distributed (Non-IID) in our setting.

We assume that each client owns heterogeneous model architecture, determined by the computation capability and training requirements of individual clients in reality.
Referring to FedGKT \cite{he2020group}, we consider the feature-driven FD framework, which can achieve heterogeneous model training while guaranteeing communication efficiency.
In this framework, the local model at each client includes a small feature extractor and a large predictor, while the global model at the server only contains a large predictor.
Let $W_e^k$ and $W_p^k$ be the feature extractor's weights and the predictor's weights of client $k$, respectively.
Moreover, we denote $W^k=\{W_e^k \cup W_p^k\}$ as the weights of the local model at client $k$, and denote $W^S$ as the weights of the global model on the server.
Let $f(W^*;\cdot)$ denote the nonlinear function determined by weights $W^\ast$, where $W^\ast\in\{\bigcup\nolimits_{i=1}^K W^k\cup W^S\}$.
In addition, we define the extracted features of client $k$ as ${H^k} = f(W_e^k;{X^k})$, the logits of client $k$ as $z_{{X^k}}^k = f(W_p^k;{H^k})$, and the logits of the server as $z_{{X^k}}^S = f({W^S};{H^k})$.
Specifically, the logits of clients are called local knowledge, and the logits of the server are called global knowledge.

The whole process of proxy-data-free FD can be divided into multiple rounds. Each round consists of two stages: local distillation, where each client updates its local model based on global knowledge transferred back from the server; global distillation, where the global model on the server performs knowledge distillation based on uploaded local knowledge from clients.
The detailed processes are illustrated as follows:

\textbf{\textit{1) Local Distillation Process}}: Each client $k$ updates its feature extractor $W_e^k$ and predictor $W_p^k$ according to the received global knowledge $z_{X^k}^S$, aiming to minimize the combination of cross-entropy loss $L_{CE}(\cdot)$ and knowledge-similarity loss $L_{sim}(\cdot)$, which can be given by:
\begin{equation}
    \label{client-distill}
    \mathop {\arg \min }\limits_{{W^k}} L_C^k: =  {L_{CE}}(p_{{X^k}}^k,{y^k}) + \beta   \cdot {L_{sim}}(p_{{X^k}}^k,p_{{X^k}}^S),
\end{equation}
where $L_C^k(\cdot)$ represents the loss function of client $k$, and $\beta$ is the hyper-parameter for weighting the effect of knowledge similarity loss. 
$p_{X^k}^S=\tau(z_{X^k}^S)$ denotes the softmax-normalized global knowledge that is broadcast to client $k$, and $p_{X^k}^k=\tau(z_{X^k}^k)$ is the softmax-normalized local knowledge from client $k$, in which $\tau(\cdot)$ is the softmax mapping.
$L_{sim}(\cdot)$ measures the similarity of normalized local and global knowledge and takes the Kullback-Leibler divergence by default.
After local training, client $k$ generates the extracted features $H^k$ and the local knowledge ${z_{{X^k}}^k}$ based on its updated feature extractor and predictor, i.e., ${{H}^k} = f(W_e^k;{X^k})$, and $z_{{X^k}}^k = f(W_p^k;{H^k})$. 
Then, client $k$ uploads its obtained features $H^k$, local knowledge $z_{{X^k}}^k$ and corresponding labels $y^k$ to the server for performing global distillation.

\textbf{\textit{2) Global Distillation Process}}:
After receiving the local knowledge from all clients, the server conducts global distillation process,
which updates the global model $W^S$ by optimizing the following objective:
\begin{equation}
\label{server-distilll}
    \mathop {\arg \min }\limits_{{W^S}} {L_S}: =  {L_{CE}}(p_{{X^k}}^S,{y^k}) + \beta  \cdot {L_{sim}}(p_{{X^k}}^S,p_{{X^k}}^k),
\end{equation}
where $L_S(\cdot)$ denotes the server-side loss function.
After distillation, the server generates global knowledge $z_{X^k}^S$ for each client $k$ using the updated global model $W^S$ and the uploaded local features $H^k$, i.e., $z_{X^k}^S=f(W^S;H^k)$. 
Then, $z_{X^k}^S$ is broadcast to client $k$. At this point, this round is completed, and the next round begins.

During the above process, only extracted features $H^k$ and local-global knowledge $\{z_{X^k}^k$, $z_{X^k}^S\}$ are exchanged between the server and client $k$.
Since the sizes of such information are significantly smaller compared with model weights, this feature-driven FD manner can achieve client-server co-distillation under model heterogeneity with slight communication overhead.

\begin{figure}[t!]
	\centering
	\includegraphics[width=0.75 \textwidth]{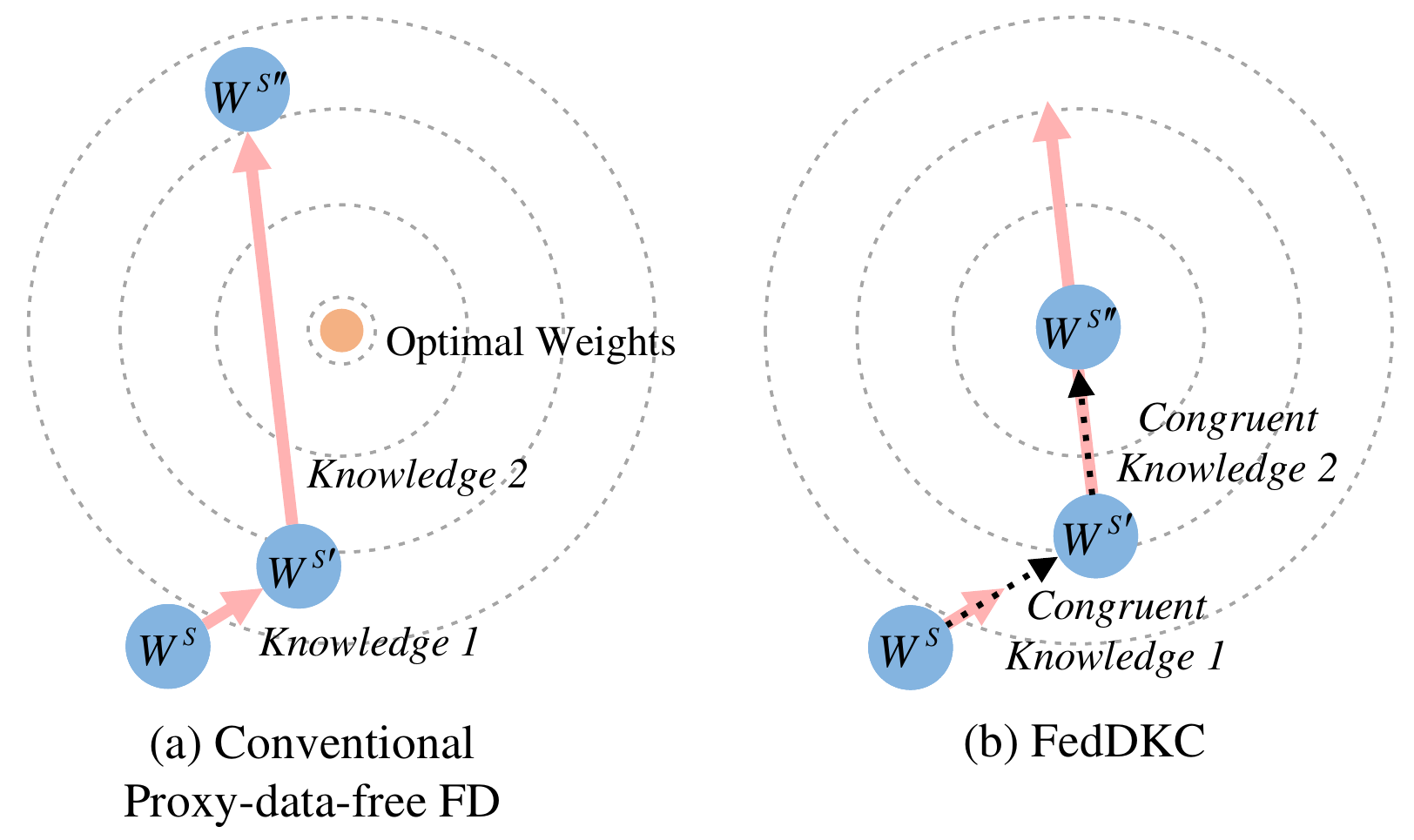}
	\caption{The effect of knowledge congruence on global model convergence.} 
	\label{weight} 
\end{figure}

\subsection{Motivation of Distributed Knowledge Congruence}
\label{motivation}
\quad \textit{\textbf{1) Existing Drawback}}:
Affected by both data heterogeneity and model heterogeneity, existing proxy-data-free FD methods are difficult to get similarly-distributed local knowledge from multiple clients.
On the one hand, data heterogeneity leads to diverse label distributions among clients, and the local model on each client tends to learn biased representations based on an independently sampled space, which favors the samples with higher frequency to promote local fitting degree. On the other hand, model heterogeneity can further exacerbate these biases since larger models tend to possess superior representation capability and generate knowledge with higher numerical differences and vice versa.

Furthermore, according to Eq. (\ref{server-distilll}), we can draw that knowledge incongruence has a non-negligible influence on server distillation since the global model needs to be optimized based on the knowledge similarity between clients and the server.
Due to the aforementioned problem, if straightforwardly learning the incongruent knowledge from clients, the server will learn an ambiguous or a biased representation and easily fail to converge smoothly, which cannot acquire approximate-optimal global knowledge and affects the training accuracy of clients in turn.
Whereas existing methods \cite{lin2020ensemble,li2019fedmd,cheng2021fedgems,itahara2020distillation,lee2021preservation,yao2021local,he2020group}, summarized in Table \ref{cmp-methods},
dismiss the ill effect of incongruent knowledge among clients, which leads to severe performance degradation.
\textbf{Fig. \ref{weight} shows the effect of knowledge congruence on global model convergence, where the red arrows indicate the direction of the negative gradient obtained by distillation on softmax-normalized local knowledge, and black arrows indicate that obtained by distillation on the refined-local knowledge.}
As shown in Fig. \ref{weight}(a), the local knowledge from a single client will contribute to an optimized direction for the global model.
However, knowledge incongruence among heterogeneous clients contributes to biased optimization and frequent fluctuation in the convergence direction.
These negative effects cause the actual result to deviate from the optimal one.

\begin{table}[t!]
	\setlength\extrarowheight{1.6pt}
	\caption{
	Comparison of FedDKC with related state-of-the-art methods. Proxy-data-free, allow model heterogeneity, efficient communication, knowledge refinement and knowledge distribution among heterogeneous clients are respectively denoted as PF, AMH, EC, KR, KDHC in this table.
	}
	\centering
	\label{cmp-methods}
	\begin{tabular}{
			>{}c |
			>{}c |
			>{}c |
			>{}c |
			>{}c |
			>{}c
		}
		\hline
		\multicolumn{1}{l|}{Method} &
		\multicolumn{1}{l|}{ PF } &
		\multicolumn{1}{l|}{ AMH }& 
            \multicolumn{1}{l|}{ EC } & 
            \multicolumn{1}{c|}{KR} & 
            \multicolumn{1}{c}{KDHC} \\ \hline
		
		FedDF \cite{lin2020ensemble} &\XSolidBrush& \CheckmarkBold & \XSolidBrush & Average&Noisy \\
		FedMD \cite{li2019fedmd}&\XSolidBrush                                               & \CheckmarkBold &\CheckmarkBold                                                                & Average                                                       & Noisy                                                                                    \\
		FedGEM \cite{cheng2021fedgems}&\XSolidBrush &\CheckmarkBold & \CheckmarkBold & None & Incongruent \\
		DS-FL \cite{itahara2020distillation} & \XSolidBrush                                            & \CheckmarkBold&\CheckmarkBold                                                                &  Entropy Reduction                                               & Incongruent                                                                              \\
		FedLSD \cite{lee2021preservation}&\CheckmarkBold& \XSolidBrush & \XSolidBrush & Soften & Incongruent \\
		FedGKD \cite{yao2021local}&\CheckmarkBold                                              & \XSolidBrush&\XSolidBrush                                                                 & Historical Information                                        & Incongruent                                                                                        \\
		FedGEN \cite{zhu2021data}&\CheckmarkBold& \CheckmarkBold & \XSolidBrush & None & Incongruent \\
		FedGKT \cite{he2020group}&\CheckmarkBold                                             & \CheckmarkBold&\CheckmarkBold                                                                & None                                                          & Incongruent                                                                              \\
		 \hline
		\textbf{FedDKC} & \textbf{\CheckmarkBold}                                     &\textbf{\CheckmarkBold} & \textbf{\CheckmarkBold}                                                       & \textbf{KKR/SKR}                               & \textbf{Congruent}                                                                       \\ \hline
	\end{tabular}
\end{table}

\textit{\textbf{2) Insight Formulation}}:
Through the above analysis, we assert that congruent local knowledge among clients is essential for optimizing the global model and realizing stabilized convergence.
Therefore, we expect to narrow the distribution differences of the original local knowledge among clients through knowledge refinement, aiming to make refined-local knowledge satisfy an approximate distribution constraint. 
Based on congruent knowledge during server-side distillation, the global model can be steadily updated towards the correct convergence direction, as shown in Fig. \ref{weight}(b).
Guided by the above insight, we propose the FedDKC algorithm, and the detailed comparison between FedDKC and related state-of-the-art methods is shown in Table \ref{cmp-methods}. 
Compared with existing proxy-data-free FD methods, our proposed FedDKC allows both model heterogeneity among clients and high communication efficiency, and pioneers to leverage knowledge congruence to promote the distillation performance.

\section{Federated Distillation based on Distributed Knowledge Congruence}
In this section, we first introduce our proposed FedDKC algorithm and its fundamental idea. 
Then, knowledge refinement strategies including kernel-based knowledge refinement (KKR) and searching-based knowledge refinement (SKR) are detailly explained. 
Finally, we provide the formal description of FedDKC.

\begin{figure*}[t]
	\centering
	\includegraphics[width=1.0 
	\textwidth]{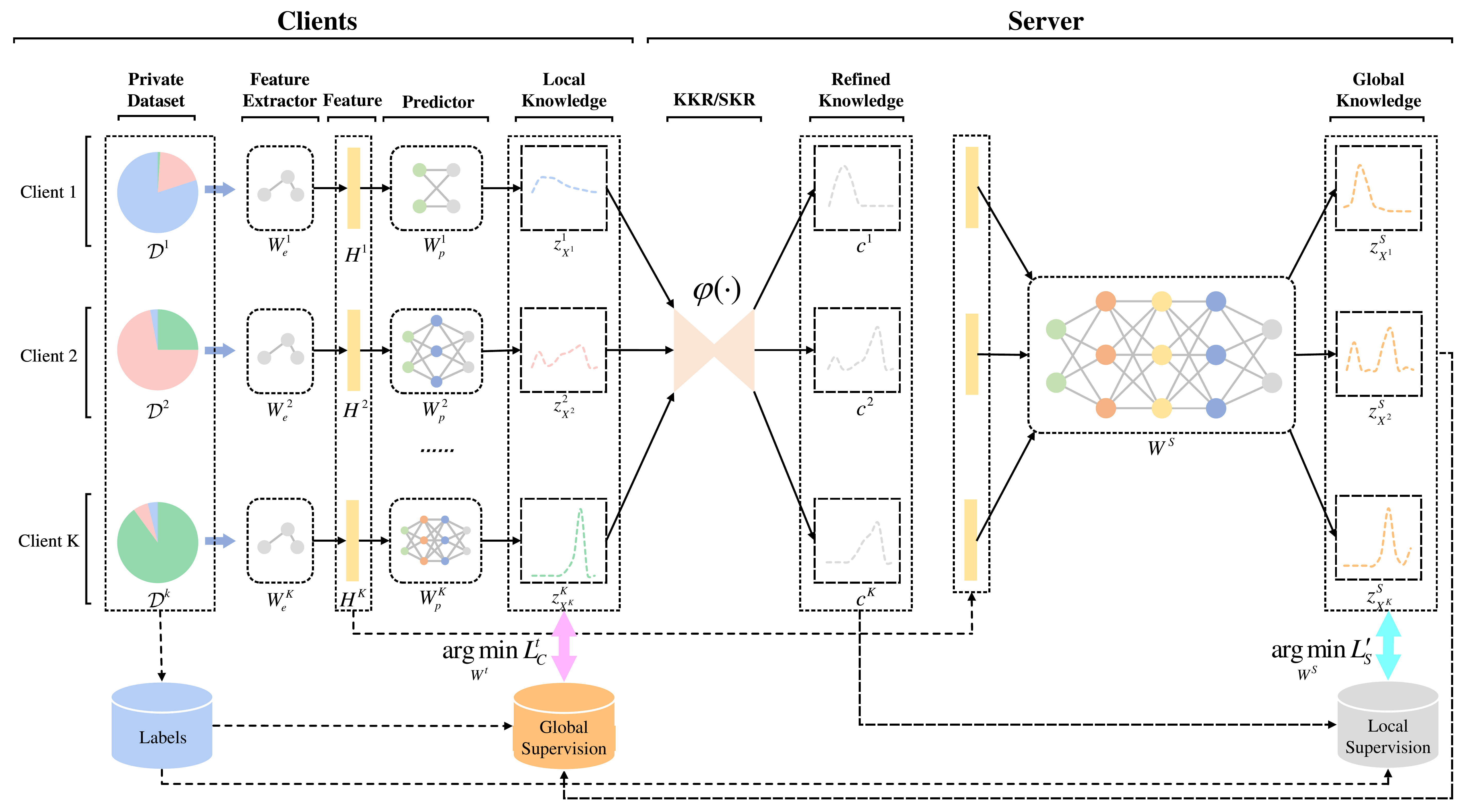}
	\caption{The overall framework of FedDKC.}
	\label{framework} 
\end{figure*}

\subsection{Framework Formulation}
\label{constrains}
Different from previous methods, we commit to achieving a tailored distribution congruence of local knowledge among clients during server-side distillation by narrowing the difference of distributed local knowledge to an acceptable constraint, as shown in Fig. \ref{framework}.
Specifically, we define $dist(\cdot)$ to measure the normalized knowledge distribution.
Taking $z^k_{X^k}$ and $z^l_{X^l}$ as inputs, they are normalized via softmax mapping $\tau(\cdot)$, and the knowledge discrepancy between client $k$ and client $l$ can be represented by $|dist(\tau(z^k_{X^k}))-dist(\tau(z^l_{X^l}))|$.
Affected by data and model heterogeneity among clients, significant discrepancy among the softmax-normalized local knowledge derived by each client is ubiquitous.
Thus, we design knowledge refinement mapping $\varphi (\cdot)$ to refine all local knowledge into a similar distribution and realize approximate congruence of local knowledge.   
Note that the local knowledge after refinement mapping is called refined-local knowledge.

Firstly, we indicate that $\varphi (\cdot)$ should satisfy the following three properties:
\begin{itemize}
    \item 
    \textbf{Probabilistic Projectivity.} For each client $k$, the refined-local knowledge is in probability space, which means that all elements in refined-local knowledge have to be non-negative and add up to 1, i.e.,
    \begin{equation}
        \varphi (z_{X^{k}}^{{k}}) \in \mathcal{P},
    \end{equation}
    where
    \begin{equation}
        \mathcal{P} = \{ Z \in {R^C} \wedge \sum\nolimits_i {{Z_i}}=1  \wedge 0 \le {Z_i} \le 1,\forall i \in \mathcal{C}\} .
    \end{equation}
    This is because the refined-local knowledge in our algorithm is required to exhibit the form of normalized, which is a necessary condition to compute similarity loss with the global knowledge.
    \item
    \textbf{Invariant Relations.} For each client's logits $z^k_{X^k} := (u_1^k,u_2^k,......,u_C^k)$, the refinement mapping $\varphi (\cdot)$ should not change the order of numeric value among all elements in $z^k_{X^k}$, i.e.,
    \begin{equation}
        \varphi (z^k_{X^k})_i \ge \varphi (z^k_{X^k})_j, \forall u_i^k \ge u_j^k,
    \end{equation}
    where $\varphi (z^k_{X^k})_i$ is the $i$-th dimension in $\varphi (z^k_{X^k})$.
    Since the structured information of local knowledge is mainly reflected in the dimensional order relations, knowledge refinement needs to maintain such relations to preserve the original information.
    \item
    \textbf{Bounded Dissimilarity.} After refining, the knowledge discrepancy between arbitrating clients should satisfy an acceptable theoretical upper bound $\varepsilon$, i.e., 
    \begin{equation}
        |dist(\varphi (z_{{X_{{k}}}}^{{k}})) - dist(\varphi (z_{{X_{{l}}}}^{{l}}))| < \varepsilon, \forall k, l\in \mathcal{K}.
    \end{equation}
    This property ensures that the refined-local knowledge is approximately congruent under the measurement of $dist(\cdot)$, which is the foundation of our motivation.
\end{itemize}

Based on the proposed knowledge refinement mapping $\varphi(\cdot)$, the new knowledge-similarity loss of the server partly depends on the refined-local knowledge among clients, which is described as follows:
\begin{equation}
	L_{sim}(p^S_{X^k},\varphi(z^k_{X^k})).
\end{equation}
As a consequence, the reformulated optimization problem with a new loss function during the global distillation process can be formulated as:
\begin{equation}
\label{server-distill}
\mathop {\arg \min }\limits_{{W^S}} {L^{'}_S}: = {L_{CE}}(p_{{X^k}}^S,{y^k}) + \beta  \cdot {L_{sim}}(p_{{X^k}}^S,\varphi (z_{{X^k}}^k)).
\end{equation}

Considering peak probability congruence and Shannon entropy congruence which are two disparate metrics to capture overall knowledge distribution, we design respective strategies for implementing knowledge refinement mapping. Specifically, kernel-based knowledge refinement (KKR) is tailored for refining the peak probability of normalized local knowledge to a customized hyper-parameter through performing a kernel-based transformation for every dimension of knowledge. Additionally, searching-based knowledge refinement (SKR) is proposed to achieve the Shannon entropy of refined-local knowledge in a given range by searching out a knowledge refinement mapping with the controlled value of output Shannon entropy. \textcolor{black}{Fig. \ref{plot} illustrates the local knowledge of two distributions extracted from samples in the TMD \cite{carpineti2018custom} dataset, where the red and blue fills respectively represent the distribution of (normalized) local knowledge from the 1st and 10th communication rounds. As displayed in Fig. \ref{plot}, the gap between the two knowledge distributions can be significantly reduced by KKR and SKR, indicating the effectiveness of distributed knowledge congruence strategy KKR and SKR in handling knowledge discrepancy.} The detailed process of our proposed strategies will be introduced in the following sections. 



\begin{figure*}[t]
	\centering
	\includegraphics[width=0.7 
	\textwidth]{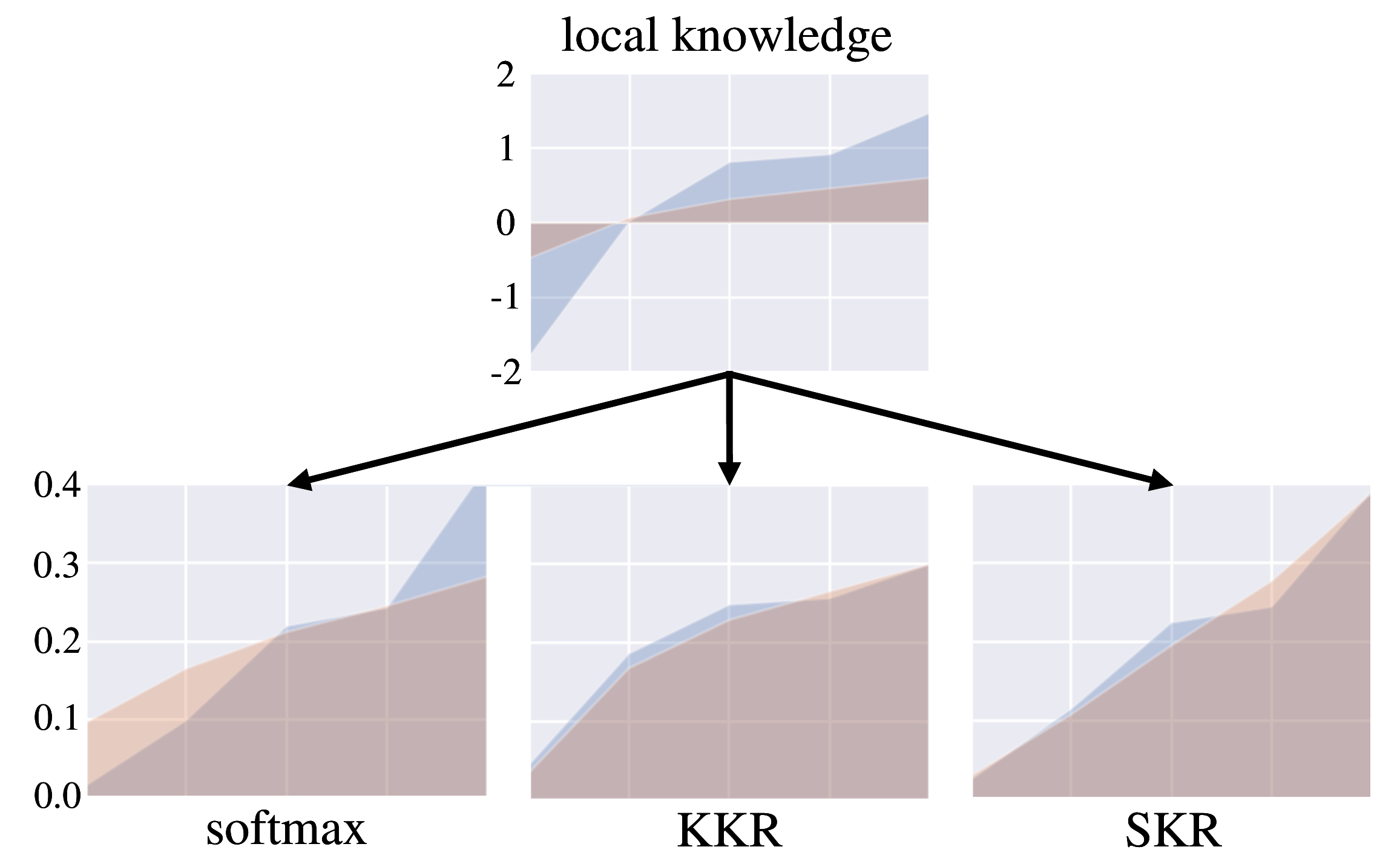}
	\caption{\textcolor{black}{Comparison of knowledge normalization with softmax, KKR and SKR over two distributions.}}
	\label{plot} 
\end{figure*}

\subsection{Kernel-based Knowledge Refinement}
\label{kn-cong}
This section proposes a kernel-based strategy to implement knowledge refinement. We adopt the maximum value in the normalized knowledge (called peak probability) to represent the distribution of the overall normalized local knowledge, \textcolor{black}{since it can reflect the model's confidence on a specific sample.}
The measurement function of knowledge distribution in KKR is defined as $dist_{KKR}(\cdot)=\max (\cdot)$, where $\max(\cdot)$ gets the maximum value of the input normalized knowledge.
To enable the peak probability congruence among clients, we require the refined peak probabilities of all clients to be a constant value $T$.

\textcolor{black}{To achieve this, we first define a non-direct-proportion and monotonically increasing kernel function $\sigma ( \cdot ), \sigma  \in \{ f|f(x) \ne k \cdot x\}  \cap \{ f|f({x_1}) - f({x_2}) \ge 0,\forall {x_1} \ge {x_2}\}$ to map each dimension of the softmax-normalized local knowledge.}
We expect that the refined-local knowledge jointly transformed from the parameterized multi-kernel functions can maintain the customized peak probability, and the parameter of kernel functions can be derived from the constraint that output peak probability equivalent to $T$.
Let $\varphi_{KKR} (z^k_{X^k})$ denote the refined-local knowledge of client $k$ derived by KKR strategy, and let ${\varphi _{KKR}}{(z_{{X^k}}^k)_i}$ denote the $i$-th dimension in $\varphi_{KKR} (z^k_{X^k})$. 
For each client $k$, $p_{X_k}^k=\tau(z^k_{X^k})$ represents its normalized knowledge, and $p^k_{X^k}:=(v_1^k,v_2^k,\cdots,v_C^k)$.
Each dimension $v_i^k$ is refined as follows: 
\begin{equation}
	{\varphi _{KKR}}{(z_{{X^k}}^k)_i} = \frac{{\sigma (\frac{v_i^k}{{t \cdot v_m^k}})}}{{\sum\limits_{j = 1}^C {\sigma (\frac{{{v_j^k}}}{{t \cdot v_m^k}})} }},
	\label{phi-init}
\end{equation}
where $m$ is the index of the empirically unique maximum value in $p^k_{X^k}$, i.e., $v_m^k=\max(p^k_{X^k})$.
Besides, $t$ represents the input scaling parameter of the kernel function $\sigma ( \cdot )$.
When $v_i^k=v_m^k$, $t$ should hold:
\begin{equation}
	\label{pk}
	\frac{{\sigma (\frac{1}{t})}}{{\sum\limits_{j = 1}^C {\sigma (\frac{{{v_j^k}}}{{t \cdot v_m^k}})} }} = T.
\end{equation}
Once $t$ is solved in Eq. (\ref{pk}), we can bring it into Eq. (\ref{phi-init}) and gain $\varphi_{KKR}(\cdot)$ as long as the properties mentioned in subsection \ref{constrains} are satisfied.
It is worth noting that there is no knowledge discrepancy among clients after refining, which means $\mid dist_{KKR}(\varphi(z_{X^k}^k))-dist_{KKR}(\varphi(z_{X^l}^l)) \mid=0$ for arbitrate clients $k$ and $l$ in this case.

To make Eq. (\ref{pk}) solvable, we further instantiate the kernel function as follows:
\begin{equation}
    \label{kernel-sample}
	\sigma(x)=kx+b, \forall k > 0,b > 0.
\end{equation}
Bringing Eq. (\ref{kernel-sample}) into Eq. (\ref{pk}), we have:
\begin{equation}
	\label{eeeee}
	\frac{{\frac{1}{t} + 1}}{{\sum\limits_{j = 1}^C {(\frac{{{v_j^k}}}{{t \cdot v_m^k}} + 1)} }} = T.
\end{equation}
Solving Eq. (\ref{eeeee}), $t$ is easily obtained as:
\begin{equation}
	t = \frac{{v_m^k - T}}{{v_m^k \cdot (CT - 1)}}.
\end{equation}
We bring $t$ into Eq. (\ref{phi-init}) to obtain the refined result of KKR strategy, which can be given by:
\begin{equation}
	{\varphi _{KKR}}{(z_{{X^k}}^k)_i} = \frac{{(CT - 1)\cdot {v_i^k} + v_m^k - T}}{C \cdot v_m^k-1}
	\label{KN-DKC-origin}.
\end{equation}

In appendix \ref{proof-theorms}, Theorem \ref{kn-less-zero} proves that the KKR strategy may project the local knowledge into a non-probability space, which indicates that one dimension in the refined-local knowledge $\varphi_{KKR} (z^k_{X^k})$ may be negative. 
Therefore, we further rectify the refined result Eq. (\ref{KN-DKC-origin}) as follows:
\begin{itemize}
    \item 
    When all dimensions in the refined-local knowledge is non-negative, i.e., $\{{\varphi _{KKR}}{(z_{{X^k}}^k)_j} \ge 0, \forall j \in \mathcal{C}\}$,  $\varphi _{KKR}(z_{{X^k}}^k)$ stays unchanged.
    \item
    When existing dimensions in $\varphi _{KKR}(z_{{X^k}}^k)$ are negative, we set the maximum dimension in $\varphi _{KKR}(z_{{X^k}}^k)$ as $T$, and let others satisfy the uniform distribution, setting as $\frac{1-T}{C-1}$.
\end{itemize}


After the above-mentioned rectification, we gain the final refined-local knowledge ${\varphi _{KKR}}(z^k_{X^k})$ via KKR strategy.
Theorem \ref{thm-1}, \ref{thm-3} and \ref{peak} prove that the KKR strategy satisfies three necessary properties mentioned in section \ref{constrains}, which is shown in appendix \ref{proof-theorms}.

\subsection{Searching-based Knowledge Refinement}
\label{searching-based}
This section proposes a searching-based strategy to implement knowledge refinement.
We adopt the Shannon entropy to indicate the distribution of normalized local knowledge, \textcolor{black}{since it integrally reflects the amount of information hidden in knowledge.}
The knowledge distribution measurement function in SKR is defined as $dist_{SKR}(\cdot) = H (\cdot)$, where $H(\cdot)$ is the Shannon entropy function.
In order to realize the Shannon entropy congruence among clients, we require that the difference between the  Shannon entropy of any refined-local knowledge and the target Shannon entropy $E$ is less than $\frac{\varepsilon }{2}$.

To this end, we define an auxiliary mapping $\psi(\theta;\cdot)$ with parameter $\theta$, to help search out an available refine mapping for SKR.
We expect that the refined-local knowledge transformed from the parameterized auxiliary mapping can satisfy the boundedness constraint of Shannon entropy differences, and the the parameter of the auxiliary mapping can be derived based on a root searching method with our given tolerance error.
Taking $z^k_{X^k}$ as input, we require $\psi(\theta;\cdot)$ to maintain numerical relationships in local knowledge unchanged, and its outputs are always in probability space, that is:
\begin{equation}
	\psi (\theta;z_{{X^k}}^k)_i \ge \psi (\theta;z_{{X^k}}^k)_j,\forall u_i^k \ge u_j^k,
\end{equation}
\begin{equation}
    \psi (\theta;z_{{X^k}}^k) \in \mathcal{P},\forall z_{{X^k}}^k,
\end{equation}
where $z^k_{X^k} := (u_1^k,u_2^k,......,u_C^k)$, and $\psi (\theta;z_{{X^k}}^k)_i$ is the $i$-th dimension in $\psi (\theta;z_{{X^k}}^k)$.
Our key idea is to search for an optimal parameter $\theta^*$ such that the difference between the refined knowledge's Shannon entropy and the target Shannon entropy $E$ is less than $\frac{\varepsilon }{2}$, which can be expressed as:
\begin{equation}
\label{search}
\begin{array}{l}
\theta^{\ast}: = \mathop {\arg \min }\limits_\theta  |H({\psi}(\theta ;z^k_{X^k})) - E|\\s.t. |H({\psi}(\theta ;z^k_{X^k})) - E| < \frac{\varepsilon }{2}.
\end{array}
\end{equation}
For client $k$, its $i$-th dimension in local knowledge $z_{X^k}^k$ is transformed via $\psi(\theta;\cdot)$, which can be given by
\begin{equation}
	\label{equ1}
	{\psi}(\theta ;z^k_{X^k})_i = \frac{{\exp (\frac{{u_i^k}}{\theta })}}{{\sum\limits_{j = 1}^C {\exp (\frac{{u_j^k}}{\theta })} }}.
\end{equation}
In this way, the searching problem of parameter $\theta$ can be converted into finding an approximate root of the following equation:
\begin{equation}
	\label{root-equ}
	{H({\psi}(\theta ;z^k_{X^k})) - E}=0,
\end{equation}
which takes $\frac{\varepsilon }{2}$ as the tolerable error.
In appendix \ref{proof-theorms}, Theorem \ref{root} prove that an approximate real root $\theta^*$ of Eq. (\ref{root-equ}) can be always figured out using the Bisection method \cite{corliss1977root}, which is also the optimal parameter that we expect to find. 
On this basis, let $\varphi_{SKR}(z^k_{X^k})$ denote the refined-local knowledge of client $k$ derived by SKR strategy, and it is defined as:
\begin{equation}
	\label{equ3}
	{\varphi _{SKR}}(z^k_{X^k}) = \psi ({\theta ^*};z^k_{X^k}).
\end{equation}
Moreover, Theorem \ref{thm-2}, \ref{shrelation} and \ref{dist-sh} prove that the SKR strategy satisfies three necessary properties mentioned in \ref{constrains}, which is shown in appendix \ref{proof-theorms}.


\begin{algorithm}[t]
    \setlength\extrarowheight{-0.4pt}
    \caption{\textbf{FedDKC}
	}
	\KwIn{
		$\{\mathcal{D}^1,\mathcal{D}^2,......,\mathcal{D}^K\}$, ${\{W^1,W^2,......,W^K\}}$
	}
	\KwOut{Trained ${\{W^1,W^2,......,W^K\}}$}
	\textbf{Initialization: }Initialize $z_{X^k}^S$ with zeros\\
	\Repeat{Reaches the number of maximum communication rounds}
	{
	    //Local Distillation Process\\
		\ForEach{$k \in \mathcal{K}$ in parallel}
		{
			\textbf{Step 1.1: }The client updates its weights based on global knowledge and local labels according to Eq. (\ref{client-distill})\\
			\textbf{Step 1.2: }The client extracts its features and local knowledge on $X^k$, that is ${H^k} = f(W_{e}^k;{X^k})$, ${z^k_{X^k}} = f({W_p^k};{H^k})$\\
	        \textbf{Step 1.3: }The client uploads $H^k$, $z^k_{X^k}$ and $y^k$ to the server\\
		}
		//Global Distillation Process\\
		\ForEach{$k \in \mathcal{K}$}
		{
				\textbf{Step 1.4: }The server computes the refined-local knowledge $c^k$ following Algorithm \ref{dkcm}.\\
				\textbf{Step 1.5: }The server updates its weights based on extracted features and local knowledge according to Eq. (\ref{server-distill})\\
				\textbf{Step 1.6: } The server generates global knowledge based on $H^k$, that is $z^S_{X^k} \leftarrow \tau(f(W^S;H^k))$\\
				\textbf{Step 1.7: }The server broadcasts $z^S_{X^k}$ to client $k$\\
		}
	}
	\textbf{Return: }Trained ${\{W^1,W^2,......,W^K\}}$\\
     \label{feddkc}
\end{algorithm}

\begin{algorithm}
    \caption{\textbf{Knowledge Refinement}
    \setlength\extrarowheight{-0.4pt}
	}
	\KwIn{$z^k_{X^k}$, $T$,  $E$}
	\KwOut{The refined-local knowledge of client $k$, denoted as $c^k$}
	\If{run \textbf{KKR}}
	{
		\textbf{Step 2.1: }Compute $c^k \leftarrow \varphi_{KKR} (z^k_{X^k})$ according to the rectified refined result of Eq. (\ref{KN-DKC-origin})
	}
	\ElseIf{run \textbf{SKR}}
	{
		\textbf{Step 2.2: }Based on Eq. (\ref{equ1}), search for an optimal $\theta^*$ by computing the approximate root of Eq. (\ref{root-equ})  using the Bisection algorithm in \cite{corliss1977root}\\
		\textbf{Step 2.3: }Compute $c^k \leftarrow \varphi_{SKR} (z^k_{X^k})$ with $\theta^*$ obtained, according to Eq. (\ref{equ1}) and Eq. (\ref{equ3})
	}
	\textbf{Return: }$c^k$
  \label{dkcm}
\end{algorithm}

\subsection{Formal Description of FedDKC}
We introduce our proposed proxy-data-free FD algorithm based on Distributed Knowledge Congruence (FedDKC) in Algorithm \ref{feddkc}, in which knowledge refinement strategy is adopted, as shown in Algorithm \ref{dkcm}.
In our algorithm, both the server and clients can perform knowledge distillation as well as knowledge generation. 
At the beginning of round $r$, each client parallelly performs local distillation jointly supervised by global knowledge and local labels (\textbf{Step 1.1}). 
It is followed by feature and knowledge extraction (\textbf{Step 1.2}). 
Then, each client uploads its extracted features, local knowledge, and corresponding labels to the server (\textbf{Step 1.3}). The server receives uploaded information from clients and refines the incongruent local knowledge (\textbf{Step 1.4}).
At this point, we can customize knowledge refinement strategies, KKR or SKR. 
The former is to be mapped according to the rectified refined result of Eq. (\ref{KN-DKC-origin}) (\textbf{Step 2.1}), and the latter needs to first search for a parameter according to Eq. (\ref{search}) (\textbf{Step 2.2}), and then refines local knowledge according to Eq. (\ref{equ1}) and Eq. (\ref{equ3}) (\textbf{Step 2.3}).
After that, feature-driven server-side distillation is conducted supervised by the refined-local knowledge along with local labels (\textbf{Step 1.5}).
After the server finishes distillation, the global knowledge is then generated based on the extracted features uploaded by clients (\textbf{Step 1.6}) and is transferred to corresponding clients (\textbf{Step 1.7}). At this point, the server and clients will start the next training round $r+1$ until model convergence.

\section{Experiments}
\label{experiments}
In this section, we provide experimental results to evaluate the performance of our proposed FedDKC algorithm, especially for verifying the accuracy improvements derived via knowledge refinement.
The detailed experiment settings are first described, and then simulation results are provided and analyzed.

\begin{figure}[t!]
	\centering
	\includegraphics[width=1 \textwidth]{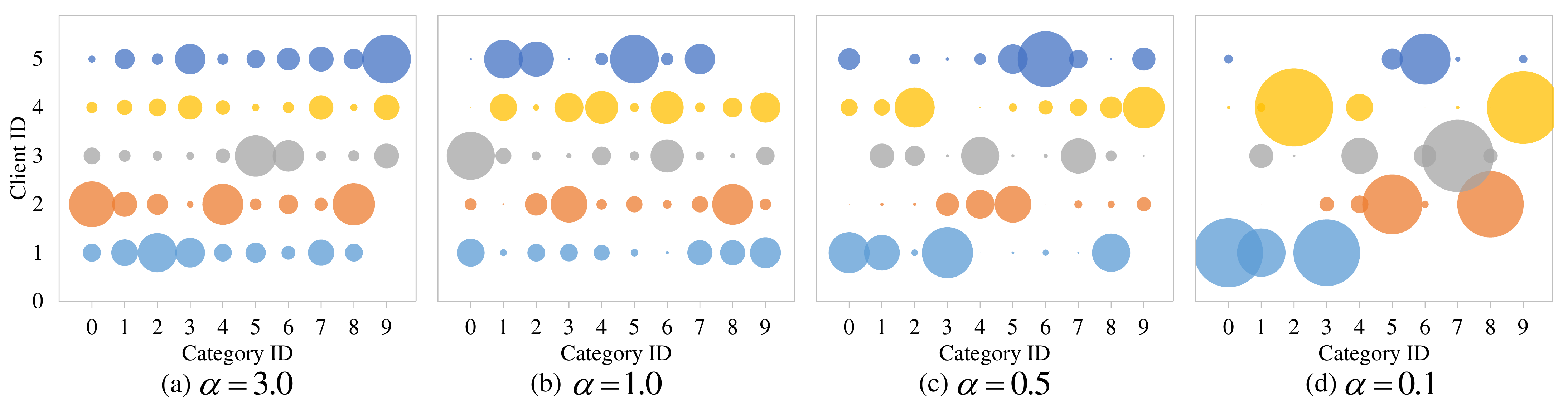}
	\caption{Visualization of data heterogeneity with different $\alpha$. Raw statistics are derived from CIFAR-10.}
	\label{data-distribution}
\end{figure}

\subsection{Experimental Setup}
\label{exp-setup}
\quad \textit{\textbf{1) Implementation and Datasets:}}
We conduct simulation experiments on a single physical server with multiple NVIDIA GeForce RTX 3090 GPU cards and enough memory space.
Our simulation code is implemented based on an open-source FL library \cite{he2020fedml}.
Besides, our training tasks are image classification on three datasets: MNIST \cite{lecun1998gradient}, CIFAR-10 \cite{cifar10} and CINIC-10 \cite{cinic10}.
We split the original dataset into five non-IID partitions and randomly distribute them to five clients. 
A hyper-parameter $\alpha$ is taken to control the degree of data heterogeneity among clients. 
Fig. \ref{data-distribution} visualizes the data distributions of clients with different $\alpha$ over CIFAR-10 dataset, in which the bubble radius indicates the samples number of a particular category in a clients' private data. 
As $\alpha$ decreases, the data distributions among clients exhibit a higher degree of heterogeneity.
In our experiments, we set $\alpha=\{0.1, 0.5, 1.0, 3.0\}$.
Before feeding data into models, we adopt commonly-used data preprocessing and augmentation strategies, including random cropping, random horizontal flipping, and normalization.

\textit{\textbf{2) Model Architecture:}}
In order to carry out model heterogeneity,
ResNet56 \cite{he2016deep} is adopted as the global model on the server; ResNet2, ResNet4, ResNet8, and ResNet10 are adopted as heterogeneous local models on five clients. 
For each local model, the feature extractor consists of the foremost Conv+Batch+ReLU+MaxPool layers, and the subsequent layers form the predictor. 
In particular, the server-side predictor is the whole global model. 
Different models exhibit a remarkable difference in terms of parameter size, memory consumption, and computation cost, as shown in Table \ref{model-hetero}.


\begin{table}[t]
\centering
\setlength\extrarowheight{0.0pt}
\caption{Configurations of models. (Taking $32\times32\times3$ as input)}
\begin{tabular}{ccccc}
\hline
\textbf{Device/Server} & \textbf{Model}& \textbf{Params (K)} & \textbf{Memory (MB)} & \textbf{Flops (M)} \\ \hline
Client 1              & ResNet2 & 0.63                & 0.31                 & 0.5                \\
Client 2             & ResNet4  & 5.18                & 1.28                 & 5.12               \\
Client 3           & ResNet8    & 10.34               & 6.93                 & 10.29              \\
Client 4/5         & ResNet10    & 9.74                & 2.17                 & 9.75               \\
Server             & ResNet56    & 577.01              & 33.79                & 87.28              \\ \hline
\end{tabular}
\label{model-hetero}
\end{table}


\color{black}
\textit{\textbf{3) Benchmarks and Criteria:}}
We compare our proposed FedDKC with state-of-art FD methods, FedGKT \cite{he2020group} and FCCL \cite{huang2022learn}. In addition, we measure the performance of the client-side models by the Top-1 and Top-5 accuracy achieved in 100 communication rounds.

\textit{\textbf{4) Hyperparameters:}} We adopt stochastic gradient descent optimizer with batch size 256, learning rate 0.03, and weight decay $5 \times {10^{ - 4}}$ for all methods. Specifically, we set the hyper-parameter for controlling the effect of knowledge similarity in loss function as $\beta=1.5$ in FedGKT and FedDKC. Besides, we leverage FashionMNIST \cite{xiao2017fashion} as the public dataset in FCCL, and follow other hyper-parameters settings in \cite{fccl-code}.
To ensure a high entropy of the refined-local knowledge in FedDKC, we set $T$ to the value that is slightly greater than $\frac{1}{C}$ and $E$ to the value that is slightly smaller than ${\log _2}C$. Precisely, we uniformly take $T$=0.11 and $E$=3.3, respectively.
\color{black}

\begin{table*}[t]
	\centering
 \setlength\extrarowheight{0.0pt}
 \setlength{\tabcolsep}{3.4pt}
	\caption{Top-1 and Top-5 accuracy on MNIST dataset. The \textbf{bold} numbers represents the best accuracy, and the {\ul underline} numbers are the second best accuracy. The same as below.
 }
\begin{tabular}{c|c|l|ccccc|c}
\hline
\textbf{\begin{tabular}[c]{@{}c@{}}Data\\ Hetero.\end{tabular}} & \textbf{Metric}                                                                & \multicolumn{1}{c|}{\textbf{Method}} & \textbf{Client 1} & \textbf{Client 2} & \textbf{Client 3} & \textbf{Client 4} & \textbf{Client 5} & \textbf{\begin{tabular}[c]{@{}c@{}}Clients\\ Avg.\end{tabular}} \\ \hline
\multirow{8}{*}{\textbf{$\alpha=3.0$}}                                 & \multirow{4}{*}{\textbf{\begin{tabular}[c]{@{}c@{}}Top-1\\ Acc.\end{tabular}}} & FedGKT                               & 30.79             & \textbf{84.88}    & \textbf{89.54}    & 92.58             & 83.66             & {\ul 76.29}                                                     \\
                                                                &                                                                                & FCCL                                 & 12.29             & 12.96             & 52.73             & 31.91             & 47.18             & 31.41                                                           \\
                                                                &                                                                                & KKR-FedDKC                           & \textbf{32.54}    & {\ul 82.28}       & 88.34             & {\ul 94.25}       & \textbf{86.13}    & \textbf{76.71}                                                  \\
                                                                &                                                                                & SKR-FedDKC                           & {\ul 32.29}       & 79.22             & {\ul 88.98}       & \textbf{94.50}    & {\ul 85.46}       & 76.09                                                           \\ \cline{2-9} 
                                                                & \multirow{4}{*}{\textbf{\begin{tabular}[c]{@{}c@{}}Top-5\\ Acc.\end{tabular}}} & FedGKT                               & 65.44             & \textbf{98.79}    & {\ul 98.74}       & 99.59             & 89.38             & 90.39                                                           \\
                                                                &                                                                                & FCCL                                 & 62.08             & 87.10             & 98.29             & 97.78             & 88.77             & 86.80                                                           \\
                                                                &                                                                                & KKR-FedDKC                           & \textbf{72.70}    & \textbf{98.79}    & \textbf{99.28}    & \textbf{99.64}    & \textbf{89.69}    & \textbf{92.02}                                                  \\
                                                                &                                                                                & SKR-FedDKC                           & {\ul 71.48}       & 97.03             & 98.66             & {\ul 99.60}       & {\ul 89.55}       & {\ul 91.26}                                                     \\ \hline
\multirow{8}{*}{\textbf{$\alpha=1.0$}}                                 & \multirow{4}{*}{\textbf{\begin{tabular}[c]{@{}c@{}}Top-1\\ Acc.\end{tabular}}} & FedGKT                               & 29.94             & {\ul 66.62}       & 73.11             & {\ul 86.78}       & 82.07             & 67.70                                                           \\
                                                                &                                                                                & FCCL                                 & 13.39             & 20.62             & 43.10             & 25.74             & 44.10             & 29.39                                                           \\
                                                                &                                                                                & KKR-FedDKC                           & \textbf{35.45}    & 62.84             & \textbf{77.84}    & \textbf{87.05}    & {\ul 87.55}       & {\ul 70.15}                                                     \\
                                                                &                                                                                & SKR-FedDKC                           & {\ul 33.58}       & \textbf{70.09}    & {\ul 75.71}       & 86.52             & \textbf{88.97}    & \textbf{70.97}                                                  \\ \cline{2-9} 
                                                                & \multirow{4}{*}{\textbf{\begin{tabular}[c]{@{}c@{}}Top-5\\ Acc.\end{tabular}}} & FedGKT                               & 70.12             & 78.76             & 87.24             & {\ul 90.27}       & 97.92             & 84.86                                                           \\
                                                                &                                                                                & FCCL                                 & 69.90             & 76.56             & 87.27             & 87.88             & 96.61             & 83.64                                                           \\
                                                                &                                                                                & KKR-FedDKC                           & {\ul 72.56}       & {\ul 78.83}       & \textbf{88.00}    & 90.26             & {\ul 99.20}       & \textbf{85.77}                                                  \\
                                                                &                                                                                & SKR-FedDKC                           & \textbf{72.60}    & \textbf{79.12}    & {\ul 87.34}       & \textbf{90.28}    & \textbf{99.39}    & {\ul 85.75}                                                     \\ \hline
\multirow{8}{*}{\textbf{$\alpha=0.5$}}                                 & \multirow{4}{*}{\textbf{\begin{tabular}[c]{@{}c@{}}Top-1\\ Acc.\end{tabular}}} & FedGKT                               & {\ul 29.95}       & {\ul 55.39}       & {\ul 58.25}       & 60.62             & 71.54             & 55.15                                                           \\
                                                                &                                                                                & FCCL                                 & 16.42             & 16.89             & 41.48             & 25.92             & 34.25             & 26.99                                                           \\
                                                                &                                                                                & KKR-FedDKC                           & \textbf{30.08}    & \textbf{55.92}    & \textbf{68.58}    & {\ul 69.10}       & \textbf{79.22}    & \textbf{60.58}                                                  \\
                                                                &                                                                                & SKR-FedDKC                           & 29.82             & 53.69             & 67.41             & \textbf{69.46}    & {\ul 78.80}       & {\ul 59.84}                                                     \\ \cline{2-9} 
                                                                & \multirow{4}{*}{\textbf{\begin{tabular}[c]{@{}c@{}}Top-5\\ Acc.\end{tabular}}} & FedGKT                               & 60.28             & 73.22             & 89.09             & 79.96             & \textbf{89.36}    & 78.38                                                           \\
                                                                &                                                                                & FCCL                                 & \textbf{70.89}    & \textbf{83.17}    & \textbf{91.60}    & \textbf{84.52}    & 86.93             & \textbf{83.42}                                                  \\
                                                                &                                                                                & KKR-FedDKC                           & {\ul 62.82}       & {\ul 76.00}       & {\ul 89.82}       & {\ul 81.84}       & 89.24             & {\ul 79.94}                                                     \\
                                                                &                                                                                & SKR-FedDKC                           & 62.52             & 72.90             & 89.62             & 79.59             & {\ul 89.28}       & 78.78                                                           \\ \hline
\multirow{8}{*}{\textbf{$\alpha=0.1$}}                                 & \multirow{4}{*}{\textbf{\begin{tabular}[c]{@{}c@{}}Top-1\\ Acc.\end{tabular}}} & FedGKT                               & 20.72             & 28.77             & 22.36             & 18.91             & 21.02             & 22.36                                                           \\
                                                                &                                                                                & FCCL                                 & 17.92             & 14.62             & \textbf{26.25}    & \textbf{19.93}    & 22.44             & 20.23                                                           \\
                                                                &                                                                                & KKR-FedDKC                           & {\ul 21.34}       & {\ul 28.86}       & {\ul 25.44}       & 18.92             & \textbf{28.13}    & \textbf{24.54}                                                  \\
                                                                &                                                                                & SKR-FedDKC                           & \textbf{21.57}    & \textbf{29.13}    & 22.94             & {\ul 18.93}       & {\ul 24.94}       & {\ul 23.50}                                                     \\ \cline{2-9} 
                                                                & \multirow{4}{*}{\textbf{\begin{tabular}[c]{@{}c@{}}Top-5\\ Acc.\end{tabular}}} & FedGKT                               & \textbf{54.79}    & 49.74             & \textbf{79.43}    & 49.20             & 50.86             & \textbf{56.80}                                                  \\
                                                                &                                                                                & FCCL                                 & {\ul 52.71}       & \textbf{51.03}    & {\ul 54.49}       & \textbf{51.50}    & \textbf{52.70}    & {\ul 52.49}                                                     \\
                                                                &                                                                                & KKR-FedDKC                           & 51.27             & {\ul 49.91}       & 49.56             & 47.02             & {\ul 51.93}       & 49.94                                                           \\
                                                                &                                                                                & SKR-FedDKC                           & 52.36             & 48.11             & {\ul 51.35}       & 49.91             & 47.07             & 49.76                                                           \\ \hline
\end{tabular}
\label{exp-mnist}
\end{table*}

\begin{table*}[t]
	\centering
 \setlength\extrarowheight{0.0pt}
	\caption{Top-1 and Top-5 accuracy on CIFAR-10 dataset.}
        \setlength{\tabcolsep}{3.4pt}	
\begin{tabular}{c|c|l|ccccc|c}
\hline
\textbf{\begin{tabular}[c]{@{}c@{}}Data\\ Hetero.\end{tabular}} & \textbf{Metric}                      & \multicolumn{1}{c|}{\textbf{Method}} & \textbf{Client 1} & \textbf{Client 2} & \textbf{Client 3} & \textbf{Client 4} & \textbf{Client 5} & \textbf{\begin{tabular}[c]{@{}c@{}}Clients\\ Avg.\end{tabular}} \\ \hline
\multirow{8}{*}{\textbf{$\alpha=3.0$}}                                 & \multirow{4}{*}{\textbf{\begin{tabular}[c]{@{}c@{}}Top-1\\ Acc.\end{tabular}}} & FedGKT                               & 27.43             & \textbf{42.69}    & 48.11             & {\ul 47.42}       & \textbf{51.98}    & 43.53                                                           \\
                                                                &                                      & FCCL                                 & 19.38             & 20.87             & 31.35             & 29.93             & 32.57             & 26.82                                                           \\
                                                                &                                      & KKR-FedDKC                           & {\ul 30.29}       & {\ul 42.64}       & {\ul 51.04}       & 45.10             & 49.26             & {\ul 43.67}                                                     \\
                                                                &                                      & SKR-FedDKC                           & \textbf{30.73}    & 44.25             & \textbf{51.98}    & \textbf{51.43}    & {\ul 50.86}       & \textbf{45.85}                                                  \\ \cline{2-9} 
                                                                & \multirow{4}{*}{\textbf{\begin{tabular}[c]{@{}c@{}}Top-5\\ Acc.\end{tabular}}} & FedGKT                               & 77.86             & 76.26             & 82.21             & 89.53             & 85.56             & 82.28                                                           \\
                                                                &                                      & FCCL                                 & 70.15             & 68.93             & 83.90             & 82.30             & 86.09             & 78.27                                                           \\
                                                                &                                      & KKR-FedDKC                           & \textbf{79.35}    & \textbf{79.43}    & {\ul 89.14}       & {\ul 90.99}       & \textbf{90.70}    & {\ul 85.92}                                                     \\
                                                                &                                      & SKR-FedDKC                           & {\ul 79.05}       & {\ul 79.39}       & \textbf{92.12}    & \textbf{93.25}    & {\ul 90.64}       & \textbf{86.89}                                                  \\ \hline
\multirow{8}{*}{\textbf{$\alpha=1.0$}}                                 & \multirow{4}{*}{\textbf{\begin{tabular}[c]{@{}c@{}}Top-1\\ Acc.\end{tabular}}} & FedGKT                               & 21.40             & 36.53             & 37.53             & \textbf{39.87}    & 35.90             & 34.25                                                           \\
                                                                &                                      & FCCL                                 & 20.20             & 22.74             & 27.67             & 28.04             & 22.07             & 24.14                                                           \\
                                                                &                                      & KKR-FedDKC                           & {\ul 26.79}       & {\ul 37.27}       & {\ul 40.85}       & 38.58             & {\ul 36.70}       & {\ul 36.04}                                                     \\
                                                                &                                      & SKR-FedDKC                           & \textbf{27.27}    & \textbf{39.53}    & \textbf{48.07}    & {\ul 38.77}       & \textbf{37.08}    & \textbf{38.14}                                                  \\ \cline{2-9} 
                                                                & \multirow{4}{*}{\textbf{\begin{tabular}[c]{@{}c@{}}Top-5\\ Acc.\end{tabular}}} & FedGKT                               & 64.97             & 78.69             & 77.403            & 71.18             & 59.54             & 70.36                                                           \\
                                                                &                                      & FCCL                                 & 67.54             & 77.89             & 80.33             & 73.65             & {\ul 63.95}       & 72.67                                                           \\
                                                                &                                      & KKR-FedDKC                           & \textbf{75.09}    & {\ul 83.18}       & \textbf{88.52}    & \textbf{79.16}    & \textbf{66.74}    & \textbf{78.54}                                                  \\
                                                                &                                      & SKR-FedDKC                           & {\ul 68.57}       & \textbf{83.33}    & {\ul 88.51}       & {\ul 77.88}       & 63.16             & {\ul 76.29}                                                     \\ \hline
\multirow{8}{*}{\textbf{$\alpha=0.5$}}                                 & \multirow{4}{*}{\textbf{\begin{tabular}[c]{@{}c@{}}Top-1\\ Acc.\end{tabular}}} & FedGKT                               & \textbf{24.23}    & 28.67             & {\ul 37.33}       & 46.06             & 35.16             & 34.29                                                           \\
                                                                &                                      & FCCL                                 & 16.68             & 24.13             & 23.82             & 29.04             & 28.24             & 24.38                                                           \\
                                                                &                                      & KKR-FedDKC                           & {\ul 24.12}       & \textbf{30.79}    & \textbf{37.97}    & {\ul 46.84}       & {\ul 37.31}       & \textbf{35.41}                                                  \\
                                                                &                                      & SKR-FedDKC                           & 24.09             & {\ul 29.10}       & 36.46             & \textbf{47.97}    & \textbf{38.50}    & {\ul 35.22}                                                     \\ \cline{2-9} 
                                                                & \multirow{4}{*}{\textbf{\begin{tabular}[c]{@{}c@{}}Top-5\\ Acc.\end{tabular}}} & FedGKT                               & 55.60             & 63.42             & 59.82             & 75.81             & 65.71             & 64.07                                                           \\
                                                                &                                      & FCCL                                 & 55.30             & 67.29    & 67.89             & {\ul 76.69}       & {\ul 71.96}       & {\ul 67.83}                                                     \\
                                                                &                                      & KKR-FedDKC                           & \textbf{56.83}    & \textbf{69.34}    & {\ul 65.11}       & 76.10             & \textbf{72.52}    & \textbf{67.98}                                                  \\
                                                                &                                      & SKR-FedDKC                           & {\ul 56.63}       & {\ul 67.82}       & 62.70             & \textbf{77.23}    & 71.88             & 67.25                                                           \\ \hline
\multirow{8}{*}{\textbf{$\alpha=0.1$}}                                 & \multirow{4}{*}{\textbf{\begin{tabular}[c]{@{}c@{}}Top-1\\ Acc.\end{tabular}}} & FedGKT                               & 20.85             & 25.38             & 34.45             & \textbf{25.11}    & 30.94             & 27.35                                                           \\
                                                                &                                      & FCCL                                 & 17.13             & 20.28             & 31.69             & 17.70             & 20.69             & 21.50                                                           \\
                                                                &                                      & KKR-FedDKC                           & {\ul 21.24}       & \textbf{27.43}    & {\ul 35.40}       & 22.68             & {\ul 31.10}       & {\ul 27.57}                                                     \\
                                                                &                                      & SKR-FedDKC                           & \textbf{21.37}    & {\ul 26.80}       & \textbf{36.37}    & {\ul 23.26}       & \textbf{35.51}    & \textbf{28.66}                                                  \\ \cline{2-9} 
                                                                & \multirow{4}{*}{\textbf{\begin{tabular}[c]{@{}c@{}}Top-5\\ Acc.\end{tabular}}} & FedGKT                               & \textbf{50.67}    & 50.00             & 50.07             & \textbf{65.49}    & 50.22             & 53.29                                                           \\
                                                                &                                      & FCCL                                 & 49.88             & \textbf{63.60}    & \textbf{63.07}    & {\ul 60.99}       & {\ul 54.13}       & \textbf{58.33}                                                  \\
                                                                &                                      & KKR-FedDKC                           & 49.05             & 50.01             & 52.19             & 58.08             & 52.05             & 52.28                                                           \\
                                                                &                                      & SKR-FedDKC                           & {\ul 49.99}       & {\ul 51.40}       & {\ul 61.48}       & 60.43             & \textbf{58.92}    & {\ul 56.44}                                                     \\ \hline
\end{tabular}
\label{exp-cifar}
\end{table*}

\begin{table*}
	\centering
	\setlength\extrarowheight{-0.5pt}
 \setlength\extrarowheight{0.0pt}
	\caption{Top-1 and Top-5 accuracy on CINIC-10 dataset.}
 \setlength{\tabcolsep}{3.4pt}
\begin{tabular}{c|c|l|ccccc|c}
\hline
\textbf{\begin{tabular}[c]{@{}c@{}}Data\\ Hetero.\end{tabular}} & \textbf{Metric}                      & \multicolumn{1}{c|}{\textbf{Method}} & \textbf{Client 1} & \textbf{Client 2} & \textbf{Client 3} & \textbf{Client 4} & \textbf{Client 5} & \textbf{\begin{tabular}[c]{@{}c@{}}Clients\\ Avg.\end{tabular}} \\ \hline
\multirow{8}{*}{\textbf{$\alpha=3.0$}}                                 & \multirow{4}{*}{\textbf{\begin{tabular}[c]{@{}c@{}}Top-1\\ Acc.\end{tabular}}} & FedGKT                               & 22.84             & 33.37             & 34.83             & 32.23             & 35.55             & 31.76                                                           \\
                                                                &                                      & FCCL                                 & 21.14             & 24.79             & 31.50             & 20.43             & 32.56             & 26.08                                                           \\
                                                                &                                      & KKR-FedDKC                           & {\ul 25.79}       & {\ul 37.86}       & {\ul 34.93}       & \textbf{39.84}    & {\ul 37.73}       & \textbf{35.23}                                                  \\
                                                                &                                      & SKR-FedDKC                           & \textbf{25.87}    & \textbf{38.28}    & \textbf{35.51}    & {\ul 37.85}       & \textbf{38.28}    & {\ul 35.16}                                                     \\ \cline{2-9} 
                                                                & \multirow{4}{*}{\textbf{\begin{tabular}[c]{@{}c@{}}Top-5\\ Acc.\end{tabular}}} & FedGKT                               & 68.49             & 73.32             & 62.92             & 80.58             & 77.94             & 72.65                                                           \\
                                                                &                                      & FCCL                                 & 74.44             & 76.56             & \textbf{69.94}    & 75.94             & 81.74             & 75.72                                                           \\
                                                                &                                      & KKR-FedDKC                           & \textbf{78.47}    & \textbf{82.17}    & 64.72             & \textbf{82.80}    & \textbf{86.22}    & \textbf{78.88}                                                  \\
                                                                &                                      & SKR-FedDKC                           & {\ul 77.68}       & {\ul 80.76}       & {\ul 65.07}       & {\ul 81.47}       & {\ul 86.02}       & {\ul 78.20}                                                     \\ \hline
\multirow{8}{*}{\textbf{$\alpha=1.0$}}                                 & \multirow{4}{*}{\textbf{\begin{tabular}[c]{@{}c@{}}Top-1\\ Acc.\end{tabular}}} & FedGKT                               & 21.08             & 27.59             & 33.50             & 22.40             & 31.97             & 27.31                                                           \\
                                                                &                                      & FCCL                                 & 19.31             & 22.44             & 30.84             & 20.40             & 24.39             & 23.48                                                           \\
                                                                &                                      & KKR-FedDKC                           & \textbf{23.72}    & \textbf{31.66}    & \textbf{34.98}    & \textbf{28.56}    & \textbf{37.62}    & \textbf{31.31}                                                  \\
                                                                &                                      & SKR-FedDKC                           & {\ul 22.73}       & {\ul 29.42}       & {\ul 34.31}       & {\ul 27.76}       & {\ul 36.05}       & {\ul 30.05}                                                     \\ \cline{2-9} 
                                                                & \multirow{4}{*}{\textbf{\begin{tabular}[c]{@{}c@{}}Top-5\\ Acc.\end{tabular}}} & FedGKT                               & 60.36             & 65.47             & 61.13             & 71.66             & 71.04             & 65.93                                                           \\
                                                                &                                      & FCCL                                 & {\ul 66.32}       & 67.42             & \textbf{68.75}    & 70.07             & 69.80             & 68.47                                                           \\
                                                                &                                      & KKR-FedDKC                           & \textbf{67.18}    & \textbf{69.11}    & {\ul 64.96}       & \textbf{76.54}    & \textbf{74.45}    & \textbf{70.45}                                                  \\
                                                                &                                      & SKR-FedDKC                           & 64.81             & {\ul 68.08}       & 63.79             & {\ul 75.10}       & {\ul 72.93}       & {\ul 68.94}                                                     \\ \hline
\multirow{8}{*}{\textbf{$\alpha=0.5$}}                                 & \multirow{4}{*}{\textbf{\begin{tabular}[c]{@{}c@{}}Top-1\\ Acc.\end{tabular}}} & FedGKT                               & 14.95             & 29.70             & 24.95             & 28.86             & 32.91             & 26.27                                                           \\
                                                                &                                      & FCCL                                 & \textbf{17.68}    & 23.85             & 27.56             & 25.05             & 25.81             & 23.99                                                           \\
                                                                &                                      & KKR-FedDKC                           & {\ul 16.24}       & \textbf{32.05}    & \textbf{26.88}    & {\ul 30.05}       & \textbf{37.77}    & \textbf{28.60}                                                  \\
                                                                &                                      & SKR-FedDKC                           & 16.02             & {\ul 31.18}       & {\ul 26.50}       & \textbf{30.63}    & {\ul 36.14}       & {\ul 28.09}                                                     \\ \cline{2-9} 
                                                                & \multirow{4}{*}{\textbf{\begin{tabular}[c]{@{}c@{}}Top-5\\ Acc.\end{tabular}}} & FedGKT                               & 58.70             & 64.42             & 64.94             & 54.39             & 71.95             & 62.88                                                           \\
                                                                &                                      & FCCL                                 & \textbf{61.45}    & 66.63             & \textbf{74.09}    & \textbf{62.20}    & 71.72             & \textbf{67.22}                                                  \\
                                                                &                                      & KKR-FedDKC                           & {\ul 60.44}       & \textbf{69.08}    & {\ul 70.04}       & 55.86             & \textbf{74.55}    & {\ul 65.99}                                                     \\
                                                                &                                      & SKR-FedDKC                           & 59.97             & {\ul 67.22}       & 69.38             & {\ul 55.98}       & {\ul 72.21}       & 64.95                                                           \\ \hline
\multirow{8}{*}{\textbf{$\alpha=0.1$}}                                 & \multirow{4}{*}{\textbf{\begin{tabular}[c]{@{}c@{}}Top-1\\ Acc.\end{tabular}}} & FedGKT                               & 21.82             & 16.79             & 21.19             & 20.83             & 19.74             & 20.07                                                           \\
                                                                &                                      & FCCL                                 & 21.12             & 17.32             & 20.69             & 19.34             & {\ul 20.66}       & 19.83                                                           \\
                                                                &                                      & KKR-FedDKC                           & \textbf{23.26}    & \textbf{23.73}    & \textbf{22.83}    & \textbf{21.57}    & \textbf{21.42}    & \textbf{22.56}                                                  \\
                                                                &                                      & SKR-FedDKC                           & {\ul 23.08}       & {\ul 23.47}       & {\ul 22.99}       & {\ul 20.96}       & 20.56             & {\ul 22.21}                                                     \\ \cline{2-9} 
                                                                & \multirow{4}{*}{\textbf{\begin{tabular}[c]{@{}c@{}}Top-5\\ Acc.\end{tabular}}} & FedGKT                               & 50.26             & 49.24             & 53.36             & 49.88             & 50.08             & 50.56                                                           \\
                                                                &                                      & FCCL                                 & 51.99             & 50.48             & \textbf{64.92}    & \textbf{53.74}    & \textbf{50.44}    & {\ul 54.31}                                                     \\
                                                                &                                      & KKR-FedDKC                           & \textbf{55.38}    & \textbf{57.15}    & {\ul 63.91}       & 50.13             & 49.97             & \textbf{55.31}                                                  \\
                                                                &                                      & SKR-FedDKC                           & {\ul 53.85}       & {\ul 56.63}       & 53.66             & {\ul 50.55}       & {\ul 50.09}       & 52.96                                                           \\ \hline
\end{tabular}
\label{exp-cinic}
\end{table*}

\color{black}
\subsection{Results}
\label{exp-results}
\quad\textit{\textbf{1) Performance Overview:}}
Table \ref{exp-mnist}, \ref{exp-cifar} and \ref{exp-cinic} display the experimental results on MNIST, CIFAR-10 and CINIC-10 datasets, respectively. Overall, our proposed FedDKC achieves superior performance than benchmark algorithms in terms of Top-1 and Top-5 accuracy on average over all datasets.
For KKR-FedDKC, the average Top-1 accuracy is improved by 1.31\% and 16.18\% compared to FedGKT and FCCL respectively, and the average Top-5 accuracy is improved by 2.55\% and 1.16\% respectively. 
For SKR-FedDKC, the average Top-1 accuracy improvements over FedGKT and FCCL are 2.28\% and 16.29\%; and the average Top-5 accuracy improvements are 2.09\% and 0.70\%, respectively.

Furthermore, we conduct comparisons on three datasets with four degrees of data heterogeneity, including a total of 120 groups of comparisons with two metrics. 
Compared with the best performance amonng FedGKT and FCCL, KKR-FedDKC and SKR-FedDKC achieve accuracy improvements in 73 and 68 groups, respectively.
Overall, our proposed FedDKC outperforms all considered benchmarks in most of the comparisons. 
Hence, we can conclude that our methods are generally applicable to improve the performance of individual clients.

\begin{figure*}[t]
	\centering
	\includegraphics[width=1 \textwidth]{
		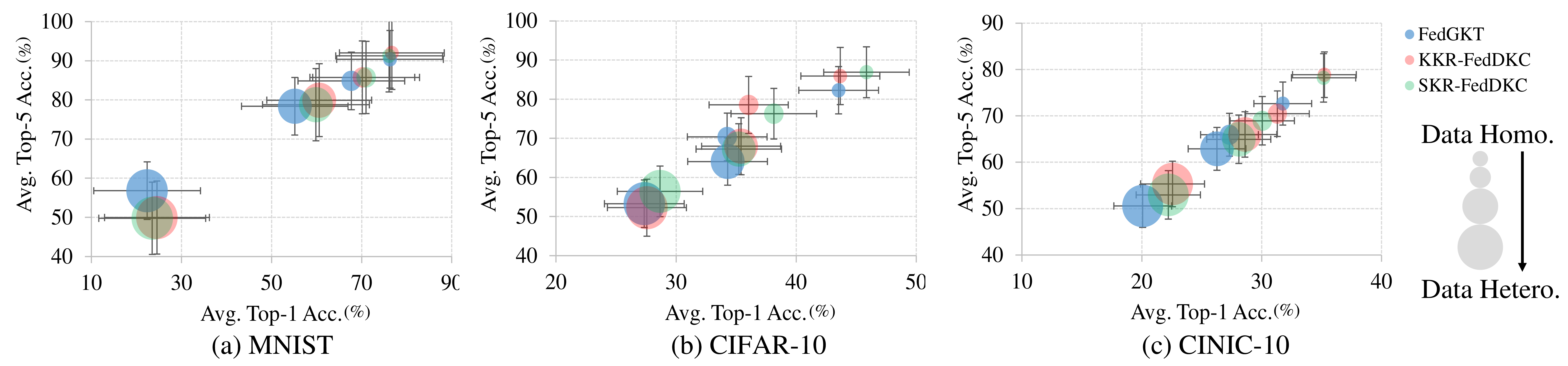}
	\caption{Average Top-1 accuracy on three datasets with various degrees of data heterogeneity.}
	\label{bubble-overview}
	\centering
	\includegraphics[width=1 \textwidth]{
		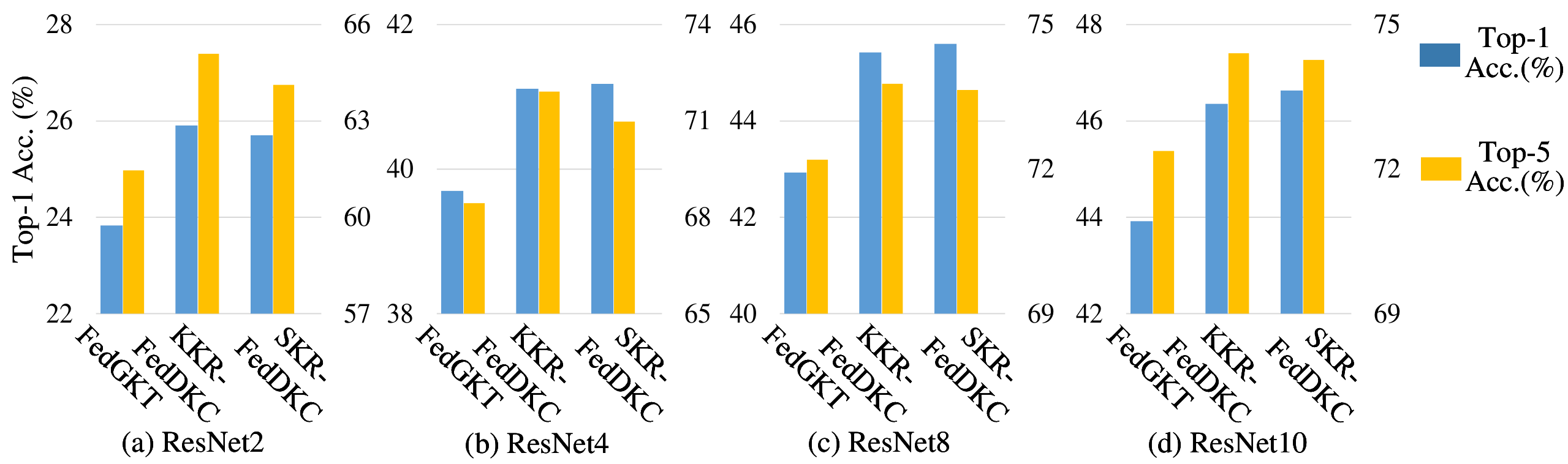}
	\caption{The average Top-1 accuracy of local models with different architectures evaluated on three datasets.}
	\label{model-diff}

 	\centering
	\includegraphics[width=1 \textwidth]{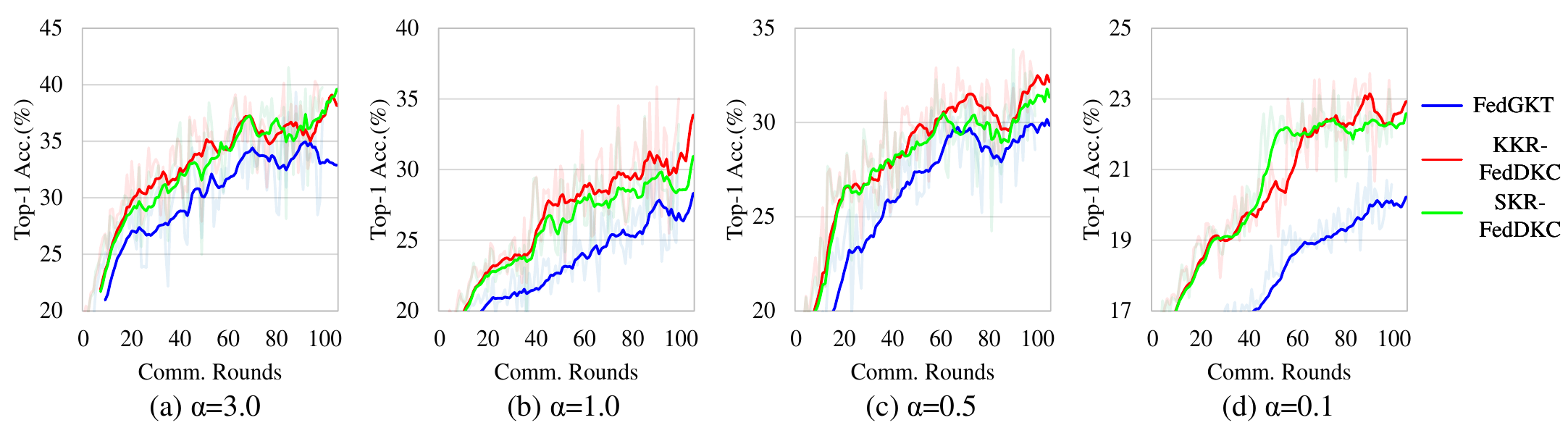}
	\caption{
		Learning curves of ResNet4 on different degrees of data heterogeneity over CINIC-10 dataset.
	}
	\label{conv-data}

 \centering
	\includegraphics[width=1 \textwidth]{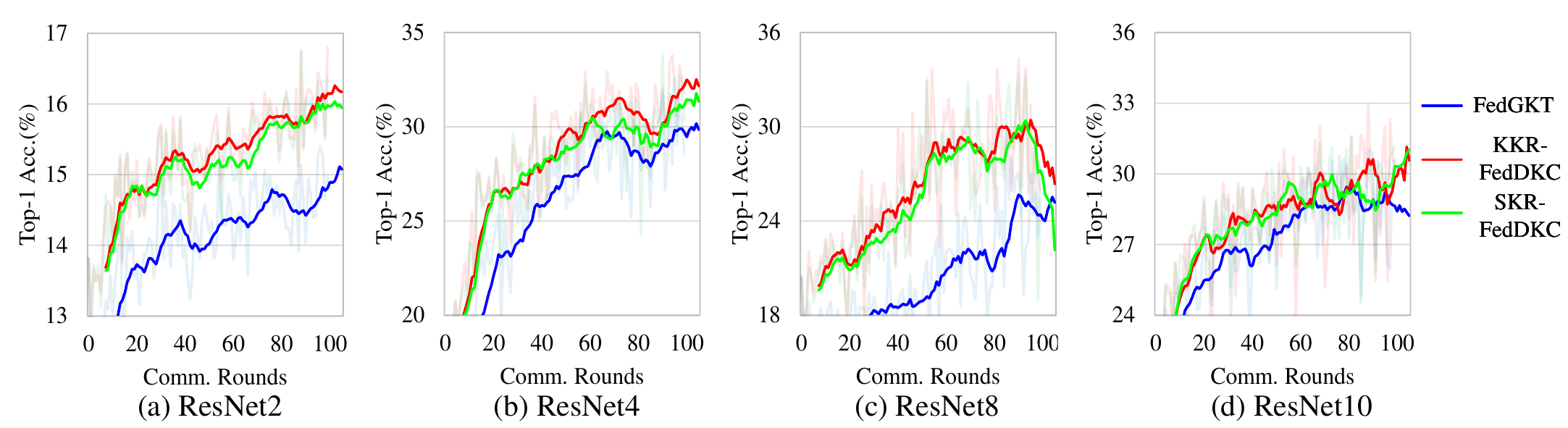}
	\caption{
		Learning curves on local models with different architectures over CINIC-10 dataset, taking $\alpha$=0.5. Results of ResNet10 are obtained from Client 4.
	}
	\label{conv-model}
 \end{figure*}

\textit{\textbf{2) Performance on Heterogeneous Data:}}
Fig. \ref{bubble-overview} compares the average accuracies of FedGKT, KKR-FedDKC, and SKR-FedDKC on different datasets under diverse degrees of data heterogeneity. As displayed, the red and green bubbles are always on the upper right of the blue bubbles for the same radius of bubbles. Hence, we can draw that FedDKC can effectively improve the general performance of clients compared with FedGKT, regardless of data heterogeneity.

\textit{\textbf{3) Performance on Heterogeneous Models:}}
Fig. \ref{model-diff} shows the comparison of the average Top-1 accuracy of local models trained with FedDKC and FedGKT on three datasets, categorized by model architectures.
We can determine that FedDKC is generally effective for local models with all kinds of architectures. 
The reason is that FedDKC can mitigate the local knowledge discrepancy during server-side distillation via KKR or SKR strategy, and thus can capture more globally-generalized representations, which will benefit client-side local distillation in turn.


\begin{figure*}[t]

\end{figure*}

\textit{\textbf{4) Communication Robustness:}}
Fig. \ref{conv-data} and Fig. \ref{conv-model} show the learning curves on different degrees of data heterogeneity and different local models, respectively. From Fig. \ref{conv-data}, we observe that FedDKC can consistently exhibit better performance than FedGKT under various data heterogeneity settings with the same number of communication rounds. 
Fig. \ref{conv-model} further confirms that FedDKC can achieve faster convergence for all heterogeneous models on clients, regardless of the knowledge refinement strategy. 
In general, compared with FedGKT, FedDKC does not increase any additional communication overhead in a single round, and can achieve faster convergence under various degrees of data heterogeneity and model architectures.
\color{black}

\begin{table}[!t]
\caption{Top-1 and Top-5 accuracy with different number of clients. Results are based on CIFAR-10 dataset, taking $\alpha=1.0$.}
\setlength{\tabcolsep}{8pt}
\setlength\extrarowheight{0.0pt}
\label{ablation}
\centering
\begin{tabular}{c|ccccc}
\hline
\multirow{2}{*}{\textbf{Method}} & \multicolumn{5}{c}{\textbf{Avg. Top-1 Acc.(\%)}}                                                                                                                                          \\ \cline{2-6} 
                                 & \multicolumn{1}{c|}{\textbf{5 Clients}} & \multicolumn{1}{c|}{\textbf{10 Clients}} & \multicolumn{1}{c|}{\textbf{20 Clients}} & \multicolumn{1}{c|}{\textbf{50 Clients}} & \textbf{Avg.}  \\ \hline
FedGKT                  & \multicolumn{1}{c|}{34.25}              & \multicolumn{1}{c|}{27.29}               & \multicolumn{1}{c|}{{\ul 26.07}}         & \multicolumn{1}{c|}{21.60}               & 29.20          \\
KKR-FedDKC              & \multicolumn{1}{c|}{{\ul 36.04}}        & \multicolumn{1}{c|}{\textbf{29.13}}      & \multicolumn{1}{c|}{\textbf{26.47}}      & \multicolumn{1}{c|}{{\ul 22.83}}         & {\ul 30.55}    \\
SKR-FedDKC              & \multicolumn{1}{c|}{\textbf{38.14}}     & \multicolumn{1}{c|}{{\ul 28.99}}         & \multicolumn{1}{c|}{25.83}               & \multicolumn{1}{c|}{\textbf{22.94}}      & \textbf{30.99} \\ \hline
\multirow{2}{*}{\textbf{Method}} & \multicolumn{5}{c}{\textbf{Avg. Top-5 Acc.(\%)}}                                                                                                                                          \\ \cline{2-6} 
                                 & \multicolumn{1}{c|}{\textbf{5 Clients}} & \multicolumn{1}{c|}{\textbf{10 Clients}} & \multicolumn{1}{c|}{\textbf{20 Clients}} & \multicolumn{1}{c|}{\textbf{50 Clients}} & \textbf{Avg.}  \\ \hline
FedGKT                  & \multicolumn{1}{c|}{70.36}              & \multicolumn{1}{c|}{63.27}               & \multicolumn{1}{c|}{\textbf{65.69}}      & \multicolumn{1}{c|}{61.31}               & 63.42          \\
KKR-FedDKC              & \multicolumn{1}{c|}{\textbf{78.54}}     & \multicolumn{1}{c|}{\textbf{68.04}}      & \multicolumn{1}{c|}{{\ul 65.04}}         & \multicolumn{1}{c|}{\textbf{63.74}}      & \textbf{65.66} \\
SKR-FedDKC              & \multicolumn{1}{c|}{{\ul 76.29}}        & \multicolumn{1}{c|}{{\ul 66.79}}         & \multicolumn{1}{c|}{64.16}               & \multicolumn{1}{c|}{{\ul 63.38}}         & {\ul 64.78}    \\ \hline
\end{tabular}
\end{table}

\color{black}
\textit{\textbf{5) Performance on Larger Number of Clients:}}
We further conduct experiments on more clients to evaluate the effectiveness of FedDKC in scenarios with larger number of clients.
Specifically, we fix the hyper-parameter $\alpha=1.0$ on the CIFAR-10 dataset, and vary the number of clients $K \in \{5,10,20,50\}$. 
Clients whose number mod 5 has a remainder of 0$\sim$4 adopt the model architectures of Client 1$\sim$5 described in Table \ref{model-hetero},
and keep other settings as described in section \ref{exp-setup}.
Thereout, we obtain the performance of FedDKC-KKR, FedDKC-SKR, and FedGKT with different numbers of clients in Table \ref{ablation}.
As displayed, although all methods achieve superior performance as the number of clients increases, FedDKC consistently outperforms FedGKT, indicating that our proposed methods can be adapted to the larger-scale FL scenarios.
\color{black}

\section{Discussions}
\subsection{Customizing Kernel Functions for KKR}
This section provides further guidance for customizing kernel functions in KKR, which can support more subtly and controllable knowledge refinement.
We give out the relaxation conditions for available kernel functions $\sigma(\cdot)$ in KKR, which are as follows:
\begin{itemize}
	\label{pro-advance}
	\item
	Non-direct-proportion
	\item
	Continuous and monotonically increasing
	\item
	Function value is consistently positive
	\item
	Parameter $t$ is solvable in Eq. (\ref{pk})
\end{itemize}
In appendix \ref{proof-theorms}, Theorem \ref{general-kn} proves that the necessary properties of refinement mapping $\varphi(\cdot)$ mentioned in section \ref{constrains} can be satisfied as long as the above relaxation conditions are met.
Up to this point, the relaxation conditions provide sufficient support for the design of feasible kernel functions: all satisfactory kernel functions can realize knowledge congruence. \textcolor{black}{On this basis, kernel functions can be flexibly customized to meet finer distribution requirements, e.g., adopting a convex kernel function to diminish the differences between classes that are not preferred by softmax-normalized knowledge, or adopting a concave kernel function to strengthen the correlation between the preferred class and the first alternative class.}




\subsection{Conversion of KKR to SKR}
\label{conver-kkr-skr}
This section discusses the feasibility of the conversion from KKR to SKR.
We observe that the $t$ in Eq. (\ref{pk}) can be derived by the searching-based method just like how $\theta^*$ in Eq. (\ref{search}) being searched out in section \ref{searching-based}. We define an auxiliary mapping $\rho(t;\cdot)$ with unknown parameter $t$. $\rho(t;z_{{X^k}}^k)_i$ denotes the $i$-th dimension in $\rho(t;z_{{X^k}}^k)$, which can be expressed as:
\begin{equation}
	\rho (t;z^k_{X^k})_i = \frac{{\sigma (\frac{{v_i^k}}{{t \cdot v_m^k}})}}{{\sum\limits_{j = 1}^C {\sigma (\frac{{v_j^k}}{{t \cdot v_m^k}})} }}.
\end{equation}
Then the optimal $t^*$ is to be searched such that the difference between the refined-local knowledge's peak probability and the target peak probability is less than a tolerable upper bound $\frac{\varepsilon }{2}$, which can be given by:
\begin{equation}
	\begin{array}{*{20}{l}}
		{{t^*}: = \mathop {\arg \min }\limits_t  |dist_{KKR} (\rho (t ;z^k_{X^k})) - T|}\\
		{s.t.|dist_{KKR} (\rho (t ;z^k_{X^k})) - T| < \frac{\varepsilon }{2}}.
	\end{array}
\end{equation}
After gaining $t^\ast$, we let:
\begin{equation}
	{\varphi _{KKR}}(z^k_{X^k}) = \rho ({t^*};z^k_{X^k}).
\end{equation}
So far, the final ${\varphi _{KKR}}(\cdot)$ is obtained. Noting that when the Bisection method \cite{corliss1977root} is adopted, the sufficient condition for available $t^*$ to be solved is that:
\begin{equation}
    \label{gen-kernel}
    h(z_{{X^k}}^k;{\epsilon _1}) \cdot h(z_{{X^k}}^k;{\epsilon _2}) < 0,\exists {\epsilon _1},{\epsilon _2},
\end{equation}
where
\begin{equation}
    h(z_{{X^k}}^k;x) = dis{t_{KKR}}(\rho (x;z_{{X^k}}^k)) - T,
\end{equation}
which is practical to satisfy. Up to this point, any kernel function satisfying Eq. (\ref{gen-kernel}) can apply to the KKR-convert-to-SKR strategy.
With the KKR to SKR conversion, our KKR can still work even when we cannot solve out $t$ from Eq. (\ref{pk}), which further promotes the customizability of kernel functions.

\subsection{Superiority and Limitations of KKR and SKR}
This section provides an analysis of the superiority and limitations of KKR and SKR. Even though section \ref{exp-results} empirically demonstrates that KKR outperforms SKR in general, the results are severely constrained by the experimental environment and the knowledge distribution metrics adopted by their respective methodologies. However, when knowledge refinement strategies apply to new data environments or improved knowledge distribution metrics are adopted, the opposite conclusion might be drawn.

According to our argument, KKR can only handle simple target knowledge distribution because it must meet to the crucial requirement that Eq. (\ref{pk}) has a solution and can be worked out. 
The analytical solution to Eq. (\ref{pk}) is frequently not available when complex kernel functions are used to satisfy the structured requirements of the target knowledge distribution (where some KKR problems can only be solved by converting to an SKR problem, as mentioned in section \ref{conver-kkr-skr}); as a result, KKR is not practical under such ordinary circumstances.
In contrast, SKR only requires that Eq. (\ref{root-equ}) has a real root, which is significantly easier to satisfy than Eq. (\ref{pk}) requested by KKR. As a result, SKR outperforms KKR in cases that require complex target knowledge distribution.

It is also worth noting that both SKR and KKR introduce computational overhead on the server side during the global distillation process, where the computation complexity of KKR is linear, and that of SKR is logarithmic (depends on the number of iterations during the parameter searching process in Bisection). Empirically, the computation costs of KKR and SKR are often affordable since they are much lower than that of the server distillation and are borne by the computation-powerful server side.

In summary, KKR is more accurate in our empirical experiments, while SKR enables more flexible setups for target knowledge distribution. 
In addition, the additional computational overhead introduced on the server side by KKR and SKR is acceptable.

\section{Related Work}

\subsection{Knowledge Distillation}
Knowledge distillation (KD) is a teacher-student learning paradigm that transfers the teacher model's knowledge to the student model through distillation.
KD has attracted much attention in ensemble model based aggregation \cite{hinton2015distilling} and cumbersome model compression \cite{wu2021spirit,romero2014fitnets,he2019knowledge,10.1007/978-3-030-82136-4_45,li2020few,peng2019correlation}.
Existing KD methods \cite{hinton2015distilling,peng2019correlation} demonstrate the feasibility that the student model learns the data-to-label representation from the teacher model. 
The subsequent work \cite{anil2018large} extends the distillation technique to exploit the potential for collaboratively optimizing a collection of models \cite{wang2021knowledge}. 
On this foundation, KD is introduced to FL for realizing collaborative training between the server and clients. 
Such distillation-based FL framework is named federated distillation (FD).

\subsection{Federated Distillation}
Typical FD methods \cite{li2019fedmd,itahara2020distillation,chang2019cronus,wu2023survey} exchange model outputs instead of model parameters among clients and the server. 
The server performs an aggregated representation of knowledge from clients and guides clients to converge toward global generalization. 
These methods, however, require a proxy dataset without exception, which is often not available during the FD process.
Recent works devote to dispensing proxy datasets through exchanging additional information, such as global models \cite{lee2021preservation,yao2021local}, generators \cite{zhu2021data}, hash values \cite{wu2023fedcache}, or extracted features \cite{he2020group,wu2023fedict}.
Parameter decentralization-based approaches \cite{lee2021preservation,yao2021local} achieve local distillation by broadcasting model parameters of the server to clients, where clients treat the downloaded global model from the server as the teacher model, and conduct local knowledge distillation based on private data.
The generator-passing-based approach \cite{zhu2021data} uses a lightweight generator to integrate information from clients, which is subsequently broadcast to clients for local training by utilizing the learned knowledge for constrained optimization. Feature-driven approaches \cite{he2020group,wu2023fedict} additionally upload client-side extracted features before global distillation and global knowledge generation.
However, none of these approaches considers that fitting local knowledge with biased distributions negatively affects the global representations under the premise of heterogeneous models among clients.

\section{Conclusion}
This paper proposes a proxy-data-free federated distillation algorithm based on distributed knowledge congruence (FedDKC). 
In our algorithm, incongruent local knowledge from distributed clients is refined to satisfy a similarly-congruent distribution without adding any communication burden. 
Furthermore, we design KKR and SKR strategies to achieve distributed knowledge congruence considering two kinds of knowledge discrepancies: the peak probability and the Shannon entropy of normalized local knowledge. 
As far as we know, this paper is the first work to boost training accuracy while maintaining communication efficiency based on distributed knowledge congruence in proxy-data-free federated distillation. 
Experiments demonstrate that FedDKC effectively improves the training accuracy of individual clients and significantly outperforms related state-of-the-art methods in various heterogeneous settings.

\section*{Acknowledgments}
We thank Prof. Lichao Sun from Lehigh University, USA, Prof. Hong Qi from Jilin University, China, Di Hou from  National University of Singapore, Singapore, Xujin Li, Hui Jiang, Zhiliu Fu, Runhan Li, Hao Tan and Prof. Zhongcheng Li from Institute of Computing Technology, Chinese Academy of Sciences, and Meicheng Liao from Shanghai Jiaotong University, China, for inspiring suggestions.

\bibliographystyle{ACM-Reference-Format}


\newpage
\appendix
\section{Appendix}
\label{proof-theorms}
\subsection{Mapping Negativity of the KKR Strategy without Rectification}
\label{mapping-neg}
\begin{theorem}
	\label{kn-less-zero}
	\textbf{There exists $\bm{{z}_{X^{k*}}^{k*}}$ such that $\bm{{\varphi _{KKR}}({z}_{X^{k*}}^{k*})_i < 0}, \exists i \in \mathcal{C}$.}
	\label{neg}
\end{theorem}

\begin{proof}
	Empirically, $p_{{X^k}}^k$ is not a uniform distribution, so there would be:
	\begin{equation}
	    v_m^k > \frac{1}{C},
	\end{equation}
	and thereout,
	\begin{equation}
		C \cdot v_m^k - 1>0.
		\label{cv}
	\end{equation}
	Also, since $T$ is the hyper-parameter that controls the peak probability of normalized knowledge, we empirically set $T>0.1$ with classification category $C\ge 10$. And hence, we have:
	\begin{equation}
		CT - 1>0.
		\label{ct}
	\end{equation}
	We let:
	\begin{equation}
		\begin{aligned}
			&\; \; \; \; \;{\varphi _{KKR}}(z^k_{X^k})_i\\
			&= \frac{{(CT - 1) \cdot v_i^k + v_m^k - T}}{{C \cdot v_m^k - 1}}\\
			&= \frac{{(CT - 1) \cdot (v_i^k + \frac{{v_m^k - T}}{{CT - 1}})}}{{C \cdot v_m^k - 1}}.
		\end{aligned}
	\end{equation}
	Accordingly, based on Eq. (\ref{ct}) and Eq. (\ref{cv}), we can infer that when:
	\begin{equation}
		\label{inequation}
		v_i^{k*} + \frac{{v_m^{k*} - T}}{{CT - 1}} < 0,
	\end{equation}
	there would be ${\varphi _{KKR}}(z_X^{k*})_i < 0$, and Eq. (\ref{inequation}) holds when $v_i^{k*} \to 0 \wedge v_m^{k*} < T$.\\
	Theorem \ref{neg} is proved.
\end{proof}

\subsection{Root Finding in the SKR Strategy}
\label{proof-solve}
\begin{theorem}
	\label{root}
	\textbf{The equation $\bm{H(\psi (\theta ;z^k_{X^k})) - E = 0}$ with unknown variable $\theta$ has a real root, and the root can be figured out using the Bisection method \cite{corliss1977root}.}
\end{theorem}
\begin{proof}
	Since $E$ is the hyper-parameter that indicates the target entropy of the refined-local knowledge, its empirical value should be taken between the Shannon entropy of the normalized local knowledge subject to a concentrated distribution and that subject to a uniform distribution, which means:
	\begin{equation}
        (C - 1) \cdot \mathop {\lim }\limits_{p \to 0^+} ( - p{\log _2}p) + \mathop {\lim }\limits_{q \to 1^-} ( - q{\log _2}q) 
        <E < C \cdot ( - \frac{1}{C}{\log _2}\frac{1}{C}),
	\end{equation}
	and that is:
	\begin{equation}
		0 < E < {\log _2}C.
	\end{equation}
	We define a continuous function $g(z^k_{X^k};\cdot)$ as follows:
	\begin{equation}
		g(z^k_{X^k};\theta ) = H(\psi (\theta ;z^k_{X^k})) - E.
	\end{equation}
	On the one hand, we have:
	\begin{equation}
		\begin{aligned}
			& \; \; \; \; {\mathop {\lim }\limits_{\theta  \to 0^+} g(z^k_{X^k};\theta )}\\
			&{ = \mathop {\lim }\limits_{\theta  \to 0^+} \sum\limits_{i = 1}^C { - \psi (\theta ;z^k_{X^k})_i \cdot {{\log }_2}\psi (\theta ;z^k_{X^k})_i}  - E}\\
			&{ = \sum\limits_{i = 1}^C { - \mathop {\lim }\limits_{\theta  \to 0^+} \psi 
			(\theta ;z^k_{X^k})_i \cdot {{\log }_2}\psi (\theta ;z^k_{X^k})_i}  - E},
		\end{aligned}
	\end{equation}
	in which
	\begin{equation}
		\begin{aligned}
			& \; \; \; \; \mathop {\lim }\limits_{\theta  \to 0^+} \psi (\theta ;z^k_{X^k})_i\\
			& = \mathop {\lim }\limits_{\theta  \to 0^+} \frac{{\exp (\frac{{u_i^k}}{\theta })}}{{\sum\limits_{j = 1}^C {\exp (\frac{{u_j^k}}{\theta })} }}\\
			& = \frac{{\mathop {\lim }\limits_{\theta  \to 0^+} \exp (\frac{{u_i^k}}{\theta })}}{{\sum\limits_{j = 1}^C {\mathop {\lim }\limits_{\theta  \to 0^+} \exp (\frac{{u_j^k}}{\theta })} }}\\
			& = \frac{{\mathop {\lim }\limits_{\theta  \to 0^+} \exp (\frac{{u_i^k}}{\theta })}}{{\mathop {\lim }\limits_{\theta  \to 0^+} \exp (\frac{{u_m^k}}{\theta })}}\\
			& = \delta (i),
		\end{aligned}
	\end{equation}
		where
	\begin{equation}
		\label{fenduan}
		\delta ({x}) = \left\{ {\begin{array}{*{20}{l}}
				{0,x = m}\\
				{1,x \ne m}
		\end{array}} \right.
		.
	\end{equation}
	Therefore, we have:
	\begin{equation}
		\begin{aligned}
					& \; \; \; {\mathop {\lim }\limits_{\theta  \to 0^+} g(z^k_{X^k};\theta )}\\
							&{ = \sum\limits_{i = 1}^C { - \mathop {\lim }\limits_{\theta  \to 0^+} \psi (\theta ;z^k_{X^k})_i \cdot {{\log }_2}\psi (\theta ;z^k_{X^k})_i}  - E}\\
							&{ = \sum\limits_{i = 1}^C {\delta (i) \cdot ( - \mathop {\lim }\limits_{x \to 0^+} x \cdot {{\log }_2}x)}  - \mathop {\lim }\limits_{x \to 1^-} x \cdot {{\log }_2}x - E}\\
				&=  - E\\
				&< 0.
		\end{aligned}
	\end{equation}
	Due to the sign preserving property of continuous functions, we can infer that there exists $ 0< {\epsilon} < 1$ making that:
	\begin{equation}
	    g(z^k_{X^k};\theta ) < 0,\forall \theta  \in [0,{\epsilon}],
	\end{equation}
	and hence,
	\begin{equation}
	    g(z^k_{X^k};\frac{{\epsilon}}{2})<0.
	\end{equation}
	On the other hand:
	\begin{equation}
		\begin{aligned}
			& \; \; \; \; {\mathop {\lim }\limits_{\theta  \to +\infty } g(z^k_{X^k};\theta )}\\
			& { = \sum\limits_{i = 1}^C { - \mathop {\lim }\limits_{\theta  \to +\infty } \psi (\theta ;z^k_{X^k})_i \cdot {{\log }_2}\psi (\theta ;z^k_{X^k})_i}  - E},
		\end{aligned}
	\end{equation}
	where
	\begin{equation}
	    \begin{aligned}
			& \; \; \; \; \; {\mathop {\lim }\limits_{\theta  \to +\infty } \psi (\theta ;z^k_{X^k})_i}\\
			&	= \mathop {\lim }\limits_{\theta  \to +\infty } \frac{{\exp (\frac{{u_i^k}}{\theta })}}{{\sum\limits_{j = 1}^C {\exp (\frac{{u_j^k}}{\theta })} }}\\
			&	= \frac{{\mathop {\lim }\limits_{\theta  \to +\infty } \exp (\frac{{u_i^k}}{\theta })}}{{\sum\limits_{j = 1}^C {\mathop {\lim }\limits_{\theta  \to +\infty } \exp (\frac{{u_j^k}}{\theta })} }}\\
			&	= \frac{1}{C}.
	    \end{aligned}
	\end{equation}
	As a consequence,
	\begin{equation}
    \begin{aligned}
		& \; \; \; \; {\mathop {\lim }\limits_{\theta  \to +\infty } g(z^k_{X^k};\theta )}\\
		&	{ = \sum\limits_{i = 1}^C { - \mathop {\lim }\limits_{\theta  \to +\infty } \psi (\theta ;z^k_{X^k})_i \cdot {{\log }_2}\psi (\theta ;z^k_{X^k})_i}  - E}\\
		&	{ = \sum\limits_{i = 1}^C { - \mathop {\lim }\limits_{\theta  \to +\infty } \psi (\theta ;z^k_{X^k})_i \cdot {{\log }_2}\mathop {\lim }\limits_{\theta  \to +\infty } \psi (\theta ;z^k_{X^k})_i}  - E}\\
		&	{ = C \cdot ( - \frac{1}{C} \cdot {{\log }_2}\frac{1}{C}) - E}\\
		&		{= {\log _2}C - E}\\
		&		>0.
		\end{aligned}
	\end{equation}
	According to the definition of limit, we can infer that for the positive real number ${\log _2}C - E \in {\bm{R^ +} }$, there exists $ M \in \bm{R^ + }$, such that:
	\begin{equation}
	     |g(z^k_{X^k};\theta ) - ({\log _2}C - E)| < {\log _2}C - E,\forall \theta  > M.
	\end{equation}
Since ${e^M} > M$, we have:
\begin{equation}
    |g(z^k_{X^k};\theta ) - ({\log _2}C - E)| < {\log _2}C - E,\forall \theta  > {e^M},
\end{equation}
and that is:
\begin{equation}
0 < g(z^k_{X^k};\theta ) < 2{\log _2}C - 2E,\forall \theta  > {e^M},
\end{equation}
and then, we have:
\begin{equation}
    g(z^k_{X^k};2{e^M}) > 0.
\end{equation}
In summary, there exists $\epsilon \in (0,1)$ and $ M \in \bm{R^ + }$ such that:
	\begin{equation}
        g(z^k_{X^k};\epsilon) \cdot g(z^k_{X^k};2{e^M}) < 0,
	\end{equation}
	in which $\frac{\epsilon}{2}<1<2{e^M}$. Hence, according to the existence theorem of zero points, $g(z^k_{X^k};\cdot)$ must have a zero in the interval $(\frac{\epsilon}{2}, 2{e^M} )$, and the zero is also the root of the equation ${H(\psi (\theta ;z^k_{X^k})) - E = 0}$.  \\
	When taking $(\frac{\epsilon}{2}, 2{e^M})$ as the input interval, $\frac{\varepsilon }{2}$ as the tolerable error, an approximate real root can be found by adopting the Bisection method \cite{corliss1977root}. Empirically, when a searching lower bound close to zero and a reasonably big searching upper bound is taken, we can always obtain an available $\theta^*$ as the approximated real root.\\
	Theorem \ref{root} is proved.
\end{proof}

\subsection{Proof of Knowledge Refinement Properties}
\label{proof-prop}
\textit{\textbf{1) Probabilistic Projectivity}}
\begin{theorem}
	\textbf{In KKR, the refined-local knowledge is in probability space.}
	\label{thm-1}
\end{theorem}

\begin{proof}
	First, we prove that $\sum\limits_{i = 1}^C {{\varphi _{KKR}}(z_{{X^k}}^k)_i = 1}$.\\
	\textbf{Case \ref{thm-1}.1.1.}\\
	When ${\varphi _{KKR}}(z^k_{X^k})_i \ge 0,\forall i \in \mathcal{C}$, we calculate the sum of all dimensions in the refined-local knowledge, which can be given by:
	\begin{equation}
	    \begin{aligned}
	        &\; \; \;\sum\limits_{i = 1}^C {\varphi _{KKR}(z^k_{X^k})_i}\\
	        &{ = \sum\limits_{i = 1}^C {\frac{{(CT - 1) \cdot {v_i^k} + v_m^k - T}}{{C \cdot v_m^k - 1}}} }\\
	        &{ = \frac{{(CT - 1) \cdot \sum\limits_{i = 1}^C {{v_i^k}} }}{{C \cdot v_m^k - 1}} + \frac{{C \cdot (v_m^k - T)}}{{C \cdot v_m^k - 1}}}.\\
	    \end{aligned}
	\end{equation}
	Since the softmax-normalized knowledge satisfies:
	\begin{equation}
		\sum {p_{X^k}^k}  = \sum\limits_{i = 1}^C {{v_i^k}}  = 1,
	\end{equation}
	hence, we have:
	\begin{equation}
		\begin{aligned}
			&\; \; \; \; \;\sum\limits_{i = 1}^C {\varphi _{KKR}(z^k_{X^k})_i}\\
			&=\frac{{(CT - 1) \cdot \sum\limits_{i = 1}^C {{v_i^k}} }}{{C \cdot v_m^k - 1}} + \frac{{C \cdot (v_m^k - T)}}{{C \cdot v_m^k - 1}}\\
			&= \frac{{CT - 1 + C \cdot (v_m^k - T)}}{{C \cdot v_m^k - 1}}\\
			&=1.
		\end{aligned}
	\end{equation}
	\textbf{Case \ref{thm-1}.1.2}\\
	When ${\varphi _{KKR}}(z^k_{X^k})_i < 0,\exists i \in \mathcal{C}$, the rectified $\varphi _{KKR}(\cdot)$ is adopted, which means:
	\begin{equation}
		\begin{aligned}
			& \; \; \; \; \;\sum\limits_{i = 1}^C {{{\varphi }_{KKR}}(z^k_{X^k})_i} \\
			&= \sum\limits_{i = 1}^C {\delta (i) \cdot \frac{{1 - T}}{{C - 1}} + T} \\
			&= (C - 1) \cdot \frac{{1 - T}}{{C - 1}} + T\\
			&= 1.
		\end{aligned}
	\end{equation}
	In summary, $\sum\limits_{i = 1}^C {{\varphi _{KKR}}(z_{{X^k}}^k)_i = 1}$ is proved.
	Then, we prove that $0\le{\varphi _{KKR}(z^k_{X^k})_i} \le 1, \forall i \in \mathcal{C}$.\\
	\textbf{Case \ref{thm-1}.2.1.}\\
	When ${\varphi _{KKR}}(z^k_{X^k})_i \ge 0,\forall i \in \mathcal{C}$, we have:
	\begin{equation}
	\begin{aligned}
		& \; \; \; \; {\varphi _{KKR}}(z^k_{X^k})_i - 1\\
		&  = {\varphi _{KKR}}(z^k_{X^k})_i - \sum\limits_{i = 1}^C {\varphi _{KKR}}(z_{{X^k}}^k)_i\\
		& { =  - \sum\limits_{j = 1}^{i - 1} {{\varphi _{KKR}}(z^k_{X^k})_j - } \sum\limits_{j = i + 1}^C {{\varphi _{KKR}}(z^k_{X^k})_j} }\\
		& { \le 0},
	\end{aligned}
	\end{equation}
	 and hence, we have $0 \le {\varphi _{KKR}}(z^k_{X^k})_i \le 1$ in this case.\\
	 \textbf{Case \ref{thm-1}.2.2.}\\
	 When ${\varphi _{KKR}}(z^k_{X^k})_i < 0,\exists i \in \mathcal{C}$, we consider the rectified form of ${\varphi _{KKR}(\cdot)}$, that is:
	 \begin{equation}
	 \label{kkr-set}
	 	{\varphi _{KKR}}(z^k_{X^k})_i \in \{T,\frac{{1 - T}}{{C - 1}}\},\forall i \in \mathcal{C}.
	 \end{equation}
	As hyper-parameter $T$ indicates the target peak probability of the refined-local knowledge, and $C$ denotes the number of classes, they empirically satisfy the following conditions:
	\begin{equation}
		\label{T-exp}
		\frac{1}{C} < T < 1,
	\end{equation}
	\begin{equation}
		\label{C-exp}
		C \ge 10 \wedge C \in \bm{Z^+},
	\end{equation}
	where $\bm{Z^+}$ is the set of positive integers. 	From Eq. (\ref{T-exp}), we have: 
	\begin{equation}
	    \label{kkr-1}
	    0 \le T \le 1.
	\end{equation}
	From Eq. (\ref{T-exp}) and Eq. (\ref{C-exp}), we can easily figure out that:
	\begin{equation}
	    1-T \le 0 \wedge C-1>0,
	\end{equation}
	and hence,
	\begin{equation}
	\label{kkr-2}
	    \frac{{1 - T}}{{C - 1}} \ge 0.
	\end{equation}
	Besides, we have:
	\begin{equation}
		\begin{aligned}
			& \; \; \;\frac{{1 - T}}{{C - 1}} - 1\\
			& = \frac{{1 - T - C + 1}}{{C - 1}}\\
			& < \frac{{ - C + 1}}{{C - 1}}\\
			& \le 0.
		\end{aligned}
	\end{equation}
	Therefore, we can get that:
	\begin{equation}
	    \label{kkr-3}
	    \frac{{1 - T}}{{C - 1}} \le 1.
	\end{equation}
	Based on Eq. (\ref{kkr-set}), Eq. (\ref{kkr-1}), Eq. (\ref{kkr-2}) and Eq. (\ref{kkr-3}), $0 \le {\varphi _{KKR}}(z^k_{X^k})_i \le 1,\forall {i \in \mathcal{C}}$ is proved.
	Combines the above two proofs, we have ${\varphi _{KKR}(z^k_{X^k})} \in \mathcal{P}$.\\
	 Theorem \ref{thm-1} is proved.
\end{proof}

\begin{theorem}
\textbf{In SKR, the refined-local knowledge is in probability space.}
	\label{thm-2}
\end{theorem}

\begin{proof}
    Define ${\varphi _{SKR}}(z_{{X^k}}^k)_i$ as the $i$-th dimension in ${{\varphi _{SKR}}(z_{{X^k}}^k)}$.
	We should first prove that $\sum\limits_{i = 1}^C {{\varphi _{SKR}}(z_{{X^k}}^k)_i} = 1$.
	\begin{equation}
		\begin{aligned}
			& \; \; \; \; \sum\limits_{i = 1}^C {{\varphi _{SKR}}(z_{{X^k}}^k)_i}\\
			& = \sum\limits_{i = 1}^C {\psi ({\theta ^*};z^k_{X^k})_i} \\
			& = \sum\limits_{i = 1}^C {\frac{{\exp (\frac{{u_i^k}}{\theta^* })}}{{\sum\limits_{j = 1}^C {\exp (\frac{{u_j^k}}{\theta^* })} }}} \\
			& = 1.
		\end{aligned}
	\end{equation}
	Then, we prove that $0 \le {\varphi _{SKR}}(z_{{X^k}}^k)_i \le 1, \forall i \in \mathcal{C}$.\\
	On the one hand, since the following inequations are always true:
	\begin{equation}
		{\exp (\frac{{u_i^k}}{\theta })} >0,
	\end{equation}
	\begin{equation}	
		{\sum\limits_{j = 1}^C {\exp (\frac{{u_j^k}}{\theta })} }>0,
	\end{equation}
	we can infer that:
	\begin{equation}
	\label{skr-1}
		{\varphi _{SKR}}(z_{{X^k}}^k)_i = \frac{{\exp (\frac{{u_i^k}}{{{\theta ^*}}})}}{{\sum\limits_{j = 1}^C {\exp (\frac{{u_j^k}}{{{\theta ^*}}})} }} > 0 \ge 0.
	\end{equation}
	On the other hand, 
	\begin{equation}
	\label{skr-2}
	\begin{aligned}
					& \; \; \; \; \;{\varphi _{SKR}}(z_{{X^k}}^k)_i\\
					& = \frac{{\exp (\frac{{u_i^k}}{{{\theta ^*}}})}}{{\sum\limits_{j = 1}^C {\exp (\frac{{u_j^k}}{{{\theta ^*}}})} }} - 1\\
					& = \frac{{\sum\limits_{j = 1}^{i - 1} {\exp (\frac{{u_j^k}}{{{\theta ^*}}}) + \sum\limits_{j = i + 1}^C {\exp (\frac{{u_j^k}}{{{\theta ^*}}})} } }}{{\sum\limits_{j = 1}^C {\exp (\frac{{u_j^k}}{{{\theta ^*}}})} }}\\
					& =  - \sum\limits_{j = 1}^{i - 1} {{\varphi _{SKR}} (z^k_{X^k})_j - \sum\limits_{j = i + 1}^C {{\varphi _{SKR}} (z^k_{X^k})_j} } \\
				    & \; { \le 0}.
	\end{aligned}
	\end{equation}
	To this end, based on Eq. (\ref{skr-1}) and Eq. (\ref{skr-2}), ${0 \le {\varphi _{SKR}}(z_{{X^k}}^k)_i \le 1},\forall i \in \mathcal{C}$ is proved.
	Combines the above two proofs, we have $\varphi_{SKR} (z^k_{X^k}) \in \mathcal{P}$.\\
	Theorem \ref{thm-2} is proved.
\end{proof}

\textit{\textbf{2) Invariant Relations}}
\begin{theorem}
	\textbf{KKR do not change the order of numeric value among all elements in local knowledge.}
	\label{thm-3}
\end{theorem}

\begin{proof}
    We first prove that the softmax mapping do not change the order of numeric value among all elements in local  knowledge.\\
    For $\forall u_i^k \ge u_j^k$, we have:
    \begin{equation}
        \begin{aligned}
            &\; \; \; \; \;v_i^k - v_j^k\\
             &= \frac{{\exp (u_i^k)}}{{\sum\limits_{l = 1}^C {\exp (u_l^k)} }} - \frac{{\exp (u_j^k)}}{{\sum\limits_{l = 1}^C {\exp (u_l^k)} }}\\
            &= \frac{{\exp (u_i^k) - \exp (u_j^k)}}{{\sum\limits_{l = 1}^C {\exp (u_l^k)} }}.
        \end{aligned}
    \end{equation}
    Since $\exp(\cdot)$ is a monotonically increasing function, there is always be:
    \begin{equation}
        \exp (u_i^k) - \exp (u_j^k) \ge 0,\forall u_i^k \ge u_j^k.
    \end{equation}
    As a result, we have:
    \begin{equation}
        \label{soft-relation}
        v_i^k \ge v_j^k,\forall u_i^k \ge u_j^k.
    \end{equation}
    Next, we need to prove that:
    \begin{equation}
    \label{kkr-re}
        {\varphi _{KKR}}(z_{{X^k}}^k)_i \ge {\varphi _{KKR}}(z_{{X^k}}^k)_j, \forall v_i^k \ge v_j^k.
    \end{equation}
    We consider the proof of Eq. (\ref{kkr-re}) in the following cases:\\
	\textbf{Case \ref{thm-3}.1.}\\
	When ${\varphi _{KKR}}(z^k_{X^k})_i \ge 0,\forall i \in \mathcal{C}$. At this point, for $\forall v_i^k \ge v_j^k$, we can infer that:
	\begin{equation}
		\begin{aligned}
		& \; \; \; \; \; \varphi_{KKR} ({z^k_{X^k}})_i - \varphi_{KKR} (z^k_{X^k})_j \\
		& = \frac{{(CT - 1) \cdot {v_i^k} + v_m^k - T}}{{C \cdot v_m^k - 1}} - \frac{{(CT - 1) \cdot {v_j^k} + v_m^k - T}}{{C \cdot v_m^k - 1}}\\ 
		& = \frac{{(CT - 1) \cdot ({v_i^k} - v_j^k)}}{{C \cdot v_m^k - 1}}.
		\end{aligned}
	\end{equation}
	With Eq. (\ref{cv}), (\ref{ct}) and the precondition ${v_i^k} \ge {v_j^k}$, we can infer that:
	\begin{equation}
	    \begin{aligned}
        &\; \; \; \; \;{\varphi _{KKR}}(z_{{X^k}}^k)_i - {\varphi _{KKR}}(z_{{X^k}}^k)_j\\
         &= \frac{{(CT - 1) \cdot (v_i^k - v_j^k)}}{{C \cdot v_m^k - 1}}\\
         &\ge 0,
\end{aligned}
	\end{equation}
	and hence, we can gain:
	\begin{equation}
	    {\varphi _{KKR}}(z^k_{X^k})_i - {\varphi _{KKR}}(z^k_{X^k})_j \ge 0,\forall v_i^k \ge v_j^k.
	\end{equation}
	\textbf{Case \ref{thm-3}.2.}\\
	When ${\varphi _{KKR}}(z^k_{X^k})_j < 0,\exists i \in \mathcal{C}$ in which ${\varphi _{KKR}}(\cdot)$ is rectified, three cases should be taken into considerations.\\
	\textbf{Case \ref{thm-3}.2.1.}\\
	When $i = m \wedge j = m$, we have:
	\begin{equation}
	    {\varphi _{KKR}}{(z_{{X^k}}^k)_i} = \varphi_{KKR} (z^k_{X^k})_j = T,
	\end{equation}
	which means ${\varphi _{KKR}}{(z_{{X^k}}^k)_i} \ge \varphi_{KKR} (z^k_{X^k})_j$ is workable.\\
	\textbf{Case \ref{thm-3}.2.2.}\\
	When $i \ne m \wedge j \ne m$, we have: 
	\begin{equation}
	    {\varphi _{KKR}}{(z_{{X^k}}^k)_i} = \varphi_{KKR} (z^k_{X^k})_j = \frac{{1 - T}}{{C - 1}},
	\end{equation}
	which means ${\varphi _{KKR}}{(z_{{X^k}}^k)_i} \ge \varphi_{KKR} (z^k_{X^k})_j$ is workable.\\
	\textbf{Case \ref{thm-3}.2.3.}\\
	When $i = m \wedge j \ne m$, following Eq. (\ref{ct}), we can infer that:
		\begin{equation}
            {\varphi _{KKR}}(z_{{X^k}}^k)_i = T = \frac{{TC - T}}{{C - 1}}
            > \frac{{1 - T}}{{C - 1}} = {\varphi _{KKR}}(z_{{X^k}}^k)_j,
			\label{t-greater}
		\end{equation}
		and ${\varphi _{KKR}}{(z_{{X^k}}^k)_i} \ge \varphi_{KKR} (z^k_{X^k})_j$ is workable as well.\\
	So far, we can prove:
	\begin{equation}
	    {\varphi _{KKR}}(z_{{X^k}}^k)_i \ge {\varphi _{KKR}}(z_{{X^k}}^k)_j,\forall v_i^k \ge v_j^k.
	\end{equation}
	Combined with Eq. (\ref{soft-relation}), we can prove that:
	\begin{equation}
	    {\varphi _{KKR}}(z_{{X^k}}^k)_i \ge {\varphi _{KKR}}(z_{{X^k}}^k)_j,\forall u_i^k \ge u_j^k.
	\end{equation}
	Theorem \ref{thm-3} is proved.
\end{proof}

\begin{theorem}
	\label{shrelation}
	\textbf{SKR do not change the order of numeric value among all elements in local knowledge.}
\end{theorem}

\begin{proof}
	For $\forall u_i^k \ge u_j^k$, we have:
	\begin{equation}
		\begin{aligned}
			& \; \; \; \; \; {\varphi _{SKR}}(z^k_{X^k})_i - {\varphi _{SKR}}(z^k_{X^k})_j\\
			& = \frac{{\exp (\frac{{u_i^k}}{{{\theta ^*}}}) - \exp (\frac{{u_j^k}}{{{\theta ^*}}})}}{{\sum\limits_{l = 1}^C {\exp (\frac{{u_l^k}}{{{\theta ^*}}})} }}.
		\end{aligned}
	\end{equation}
	Since$\frac{{u_i^k}}{{{\theta ^*}}} \ge \frac{{u_j^k}}{{{\theta ^*}}}$,we have:
	\begin{equation}
		\exp (\frac{{u_i^k}}{{{\theta ^*}}}) - \exp (\frac{{u_j^k}}{{{\theta ^*}}}) \ge 0.
	\end{equation}
	Hence,
	\begin{equation}
			{\frac{{\exp (\frac{{u_i^k}}{{{\theta ^*}}}) - \exp (\frac{{u_j^k}}{{{\theta ^*}}})}}{{\sum\limits_{l = 1}^C {\exp (\frac{{u_l^k}}{{{\theta ^*}}})} }}}
		 \ge 0.
	\end{equation}
	In summary, we can always get ${\varphi _{KKR}}{(z_{{X^k}}^k)_i} \ge \varphi_{KKR} (z^k_{X^k})_j$ when ${\forall {v_i^k} \ge {v_j^k}}$. \\
	Theorem \ref{shrelation} is proved.
\end{proof}

\textit{\textbf{3) Bounded Dissimilarity}}
\begin{theorem}
	\textbf{After refining by KKR, the knowledge discrepancy between arbitrating clients should satisfy an acceptable theoretical upper bound $\bm{\varepsilon}$.}
	\label{peak}
\end{theorem}

\begin{proof}
	We first prove that the peak probability of the knowledge refined by KKR is always $T$. Two cases are taken into consideration.\\
	\textbf{Case 6.1.}\\
	When ${\varphi _{KKR}}(z^k_{X^k})_i \ge 0,\forall i \in \mathcal{C}$, according to Theorem \ref{thm-3}, we have:
	\begin{equation}
	\label{rkdc}
		\begin{aligned}
			& \; \; \; \; \max ({\varphi _{KKR}}(z^k_{X^k}))\\
			& = \max ({\varphi _{KKR}}(z^k_{X^k})_1,{\varphi _{KKR}}(z^k_{X^k})_2,\\
			& \; \; \; \; \; ......,{\varphi _{KKR}}(z^k_{X^k})_C)\\
			&  = {\varphi _{KKR}}(z^k_{X^k})_m\\
			&  = \frac{{(CT - 1) \cdot v_m^k + v_m^k - T}}{{C \cdot v_m^k - 1}}\\
			&  = T.
		\end{aligned}
	\end{equation}
	\textbf{Case 6.2.}\\
	When ${\varphi _{KKR}}(z^k_{X^k})_j < 0,\exists i \in \mathcal{C}$, 
	we can conduct the following inference based on Eq. (\ref{t-greater}):
	\begin{equation}
	\begin{aligned}
			& \; \; \; \; \max ({\varphi _{KKR}}(z^k_{X^k}))\\
			& =\max (T,\frac{{1 - T}}{{C - 1}})\\
			&  = T.
	\end{aligned}
	\end{equation}
	So far, for $\forall {k_1},{k_2} \in \mathcal{K}$, we have:
	\begin{equation}
		\begin{aligned}
			& \; \; \; \; \; |dis{t_{KKR}}(\varphi_{KKR} (z_X^{{k_1}})) - dis{t_{KKR}}(\varphi_{KKR} (z_X^{{k_2}}))|\\
			& = |\max (\varphi_{KKR} (z_X^{{k_1}})) - \max (\varphi_{KKR} (z_X^{{k_2}}))|\\
			& = |T - T|\\
			& = 0\\
			& < \varepsilon .
		\end{aligned}
	\end{equation}

	Theorem \ref{peak} is proved.
\end{proof}

\begin{theorem}
    \label{dist-sh}
	\textbf{After refining by SKR, the knowledge discrepancy between arbitrating clients should satisfy an acceptable theoretical upper bound $\bm{\varepsilon}$.}
\end{theorem}

\begin{proof}
	Since we cannot provide $\varphi_{SKR}(\cdot)$ directly, our demonstration is to follow two steps: 
	\begin{itemize}
	    \item [(a)]
	    To prove that $\varphi_{SKR}(\cdot)$ is able to be constructed according to section \ref{searching-based}.
	    \item [(b)]
	    To prove that the knowledge discrepancy between arbitrate clients should satisfy an acceptable theoretical upper bound $\varepsilon$ after refining the local knowledge by the available SKR.
	\end{itemize}
	To prove step (a), we should first search for an optimal $\theta^*$ in $\psi (\theta ; \cdot )$ just as mentioned in Eq. (\ref{search}) and Eq. (\ref{equ1}). Furthermore, our problem is converted into finding the root of Eq. (\ref{root-equ}), whose availability has been proved in Theorem \ref{proof-solve}.

    To prove step (b), we calculate the differences in knowledge distributions based on metric $dist_{SKR}(\cdot)$, in that for $\forall {{k}_1},{k_2} \in \mathcal{K}$,
	\begin{equation}
	    \begin{aligned}
& \; \; \; \; \;|dis{t_{SKR}}({\varphi _{SKR}}(z_X^{{k_1}})) - dis{t_{SKR}}({\varphi _{SKR}}(z_X^{{k_2}}))|\\
& = |H({\varphi _{SKR}}(z_X^{{k_1}})) - H({\varphi _{SKR}}(z_X^{{k_2}}))|\\
& = |(H({\varphi _{SKR}}(z_X^{{k_1}})) - E) - (H({\varphi _{SKR}}(z_X^{{k_2}})) - E)|\\
 & \le |(H({\varphi _{SKR}}(z_X^{{k_1}})) - E)| + |(H({\varphi _{SKR}}(z_X^{{k_2}})) - E)|\\
 &< \frac{\varepsilon }{2} + \frac{\varepsilon }{2}\\
 &= \varepsilon. \\
\end{aligned}
	\end{equation}
Theorem \ref{dist-sh} is proved.
\end{proof}

\subsection{Sufficient Conditions for Available Kernel functions in the KKR strategy}
\begin{theorem}
\label{general-kn}
\textbf{The constructed KKR can satisfy all properties mentioned in section \ref{constrains} as long as the kernel function $\bm{\sigma(\cdot)}$ satisfies the following relaxation conditions:
\begin{itemize}
    \label{pro-advance}
    \item[(a)]
    None-direct-proportion, i.e. $\sigma  \notin \{ f|f(x) = k \cdot x\}$
    \item[(b)]
	Continuous and monotonically increasing, i.e. $\sigma  \in \{ f|\mathop {\lim }\limits_{x \to c} f(x) = f(c)\}  \cap \{ f|f({x_1}) - f({x_2}) \ge 0,\forall {x_1} \ge {x_2}\}$
	\item[(c)]
	Function value is consistently positive, i.e. $\sigma  \in \{ f|f(x) > 0\}$
	\item[(d)]
	Parameter $\bm{t}$ is solvable in Eq. (\ref{pk}), i.e. $(\frac{{\sigma (\frac{1}{{{t_1}}})}}{{\sum\limits_{j = 1}^C {\sigma (\frac{{v_j^k}}{{{t_1} \cdot v_m^k}})} }} - T) \cdot (\frac{{\sigma (\frac{1}{{{t_2}}})}}{{\sum\limits_{j = 1}^C {\sigma (\frac{{v_j^k}}{{{t_2} \cdot v_m^k}})} }} - T) < 0,\exists {t_1},{t_2}$
\end{itemize}}
\end{theorem}
\begin{proof}
To prove the necessary properties in section \ref{constrains}, we first introduce a lemma to confirm that the kernel function scaling parameter $t$ is consistently positive.\\
\textbf{Lemma.}
\textbf{\textit{When the kernel function satisfies the relaxation conditions mentioned in Theorem \ref{pro-advance}, $\bm{t}$ is consistently positive.}}\\
\textit{Proof of Lemma. }
We first claim that $t \ne 0$ as an denominator in Eq. (\ref{pk}).
Then we prove that $t<0$ can never hold. According to Eq. (\ref{T-exp}) and condition (b), we can infer that:
\begin{equation}
\begin{aligned}
& \; \; \; \; \; \frac{{\sigma (\frac{1}{t})}}{{\sum\limits_{j = 1}^C {\sigma (\frac{{v_j^k}}{{t \cdot v_m^k}})} }}\\
 &< \frac{{\sigma (\frac{1}{t})}}{{\sum\limits_{j = 1}^C {\sigma (\frac{{v_j^k}}{{t \cdot v_j^k}})} }}\\
 &= \frac{{\sigma (\frac{1}{t})}}{{C \cdot \sigma (\frac{1}{t})}}\\
 &= \frac{1}{C}\\
 &< T,
\end{aligned}
\end{equation}
which indicates:
\begin{equation}
\label{ineqq}
    \frac{{\sigma (\frac{1}{t})}}{{\sum\limits_{j = 1}^C {\sigma (\frac{{v_j^k}}{{t \cdot v_m^k}})} }} \ne T,
\end{equation}
and Eq. (\ref{ineqq}) is in conflict with Eq. (\ref{pk}). Hence, we can never take $t \le 0$ when relaxation conditions in Theorem \ref{pro-advance} satisfy. While condition (d) indicates that we can always solve out a $t$, there should always be $t>0$.
\\
Lemma is proved.\\
So far, we begin to prove the necessary properties mentioned in section \ref{constrains}.

\textit{\textbf{1) Probabilistic Projectivity:}}
As stated in condition (b), i.e. $\sigma (x) > 0,\forall x \in R$, hence, we have:
\begin{equation}
{\varphi _{KKR}}(z^k_{X^k})_i = \frac{{\sigma (\frac{{v_i^k}}{{t \cdot v_m^k}})}}{{\sum\limits_{j = 1}^C {\sigma (\frac{{v_j^k}}{{t \cdot v_m^k}})} }} > 0.
\end{equation}

What is more, 
\begin{equation}
\begin{aligned}
& \; \; \; \; \sum\limits_{i = 1}^C {{\varphi _{KKR}}(z^k_{X^k})_i} \\
& = \sum\limits_{i = 1}^C {\frac{{\sigma (\frac{{v_i^k}}{{t \cdot v_m^k}})}}{{\sum\limits_{j = 1}^C {\sigma (\frac{{v_j^k}}{{t \cdot v_m^k}})} }}} \\
& = 1.
\end{aligned}
\end{equation}
Hence, we prove ${\varphi _{KKR}}(z^k_{X^k}) \in \mathcal{P}$.

\textit{\textbf{2) Invariant Relations:}}
As $v_i^k \ge v_j^k$, $t>0$ and $v^k_m>0$, we can infer that:
\begin{equation}
    \frac{{v_i^k}}{{t \cdot v_m^k}} \ge \frac{{v_j^k}}{{t \cdot v_m^k}}.
\end{equation}
 Consequently, we have:
\begin{equation}
\begin{aligned}
& \; \; \; \; \; \forall v_i^k \ge v_j^k,\\
& \; \; \; \; \; {\varphi _{KKR}}(z^k_{X^k})_i - {\varphi _{KKR}}(z^k_{X^k})_j\\
& = \frac{{\sigma (\frac{{v_i^k}}{{t \cdot v_m^k}}) - \sigma (\frac{{v_j^k}}{{t \cdot v_m^k}})}}{{\sum\limits_{j = 1}^C {\sigma (\frac{{v_j^k}}{{t \cdot v_m^k}})} }}\\
& \ge 0.
\end{aligned}
\end{equation}
Referencing to the process in proving Eq. (\ref{soft-relation}), we can summarize that:
	\begin{equation}
	    {\varphi _{KKR}}(z_{{X^k}}^k)_i \ge {\varphi _{KKR}}(z_{{X^k}}^k)_j,\forall u_i^k \ge u_j^k.
	\end{equation}

	\textit{\textbf{{3) Bounded Dissimilarity:}}}
Based on the property invariant relations \textit{\textbf{2)}} in Theroem \ref{general-kn}, the proven detail is just the same as Eq. (\ref{rkdc}).\\
Theroem \ref{general-kn} is proved.
\end{proof}
\end{document}